\documentclass{article}
\usepackage[utf8]{inputenc}
\usepackage[english]{babel}

\usepackage{multirow}

\newcommand{\ii}{\boldsymbol{i}}
\newcommand{\sign}{\text{sign}}

\newcommand{\re}[1]{\text{Re}\left\{#1\right\}}
\newcommand{\im}[1]{\text{Im}\left\{#1\right\}}

\usepackage{graphicx}
\usepackage{amsmath,amssymb,amsfonts,amsthm}
\usepackage{textcomp}
\usepackage{mathtools}
\usepackage{float}
\usepackage{subscript}
\usepackage{booktabs}
\usepackage{amsthm}
\usepackage{multirow}
\usepackage{makecell}
\usepackage{enumerate}
\usepackage{amsmath}
\newtheorem{Definition}{Definition}
\newtheorem{Theorem}{Theorem}

\theoremstyle{definition}
\newtheorem{Example}{Example}

\usepackage{authblk}

\usepackage{hyperref}
\usepackage{xcolor}
\newcommand{\new}[1]{{#1}}

\author[1]{Rama Murthy Garimella\footnote{Email: rama.murthy@mahindrauniversity.edu.in, ORCID: 0000-0002-5669-1846}}

\author[2]{Marcos Eduardo Valle\footnote{Email: valle@ime.unicamp.br, ORCID: 0000-0003-4026-5110}}

\author[2]{Guilherme Vieira\footnote{Email: vieira.g@ime.unicamp.br, ORCID: 0000-0003-3361-6154}}

\author[3]{Anil Rayala\footnote{Email: anil.rayala@students.iiit.ac.in, ORCID: 0000-0002-2836-1233}}

\author[4]{Dileep Munugoti\footnote{Email: d.munugoti@iitg.ernet.in, ORCID: 0000-0002-6830-580X}}

\affil[1]{Mahindra University, Ecole Centrale School of Engineering, Hyderabad, India.}

\affil[2]{Universidade Estadual de Campinas (UNICAMP), Campinas, Brazil}

\affil[3]{International Institute of Information Technology, Hyderabad, India.}

\affil[4]{Indian Institute of Technology, Guwahati, India}

\title{Dynamics of Structured Complex-Valued Hopfield Neural Networks\thanks{This work was supported in part by the National Council for Scientific and Technological Development (CNPq) under grant no 315820/2021-7, the S\~ao Paulo Research Foundation (FAPESP) under grant no 2023/03368-0, and the Postdoctoral Researcher Program (PPPD) at the Universidade Estadual de Campinas (UNICAMP).}
}

\begin{document}

\maketitle

\begin{abstract}
\new{
In this paper, we explore the dynamics of structured complex-valued Hopfield neural networks (CvHNNs), which arise when the synaptic weight matrix possesses specific structural properties. We begin by analyzing CvHNNs with a Hermitian synaptic weight matrix and establish the existence of four-cycle dynamics in CvHNNs with skew-Hermitian weight matrices operating synchronously. Furthermore, we introduce two new classes of complex-valued matrices: braided Hermitian and braided skew-Hermitian matrices. We demonstrate that CvHNNs utilizing these matrix types exhibit cycles of length eight when operating in full parallel update mode. Finally, we conduct extensive computational experiments on synchronous CvHNNs, exploring other synaptic weight matrix structures. The findings provide a comprehensive overview of the dynamics of structured CvHNNs, offering insights that may contribute to developing improved associative memory models when integrated with suitable learning rules.
}  

\vspace{1em}
\end{abstract}

\section{Introduction} \label{sec:introduction}

Artificial neural networks have been conceived as emulators of the biological neural network synapse process. Their processing units, the artificial neurons, usually act based on input signals received from other neurons or cells. Like a biological neuron firing an electric impulse in the presence of specific chemical components in appropriate concentrations, an artificial neuron fires when certain mathematical conditions are satisfied. The first artificial neuron model, referred to as the threshold neuron, is credited to McCulloch and Pitts \cite{mcculloch43}.  The threshold neuron is activated if a weighted sum of the inputs is greater than or equal to a certain threshold; otherwise, it does not activate. The threshold neuron was thoroughly experimented with and developed upon, giving birth to many neural network models featuring non-linearities, hidden layers, recurrent connections, alternative learning rules, and many more configurations. However, all the parameters and inputs in the threshold neuron and derived models were set to real numbers. This restriction was eventually lifted by Aizenberg and collaborators, creating the branch of complex-valued neural networks \cite{aizenberg11book,Murthy2004Complex-valuedHypercube}. 

Hopfield neural networks (HNNs) are one of the earlier models based on the threshold neuron responsible for the resurgent interest in artificial neural networks in the 1980s \cite{hopfield1982neural}. \new{Besides being a well-established model for associative memories, these models find many applications, including pattern recognition and diagnosis \cite{Alenezi2025Pixel-guidedDiagnosis}, matrix factorization \cite{Li2016,Li2024BinaryOptimization}, time series prediction \cite{AbouHassan2024}, and combinatorial optimization problems \cite{Serpen_2008,Pajares2010,Uykan2020OnOptimization}. Furthermore, they are closely related to the attention mechanism used in transformer architecture \cite{castro20nn,Krotov2020LargeLearning,Ramsauer2020HopfieldNeed,karakida2024hierarchicalassociativememoryparallelized}.} Briefly, an HNN consists of $N$ bipolar threshold neurons that are fully interconnected. A single bipolar threshold neuron outputs either $+1$ or $-1$, and the network state can be interpreted as a vertex in the unit hypercube $\{-1,+1\}^N$. The Hopfield model is a recurrent neural network since each neuron receives inputs from all other neurons. An HNN is often associated with a directed graph, with the nodes being the neurons and the edges corresponding to the synaptic connections. In the original Hopfield model, the synaptic weights are arranged in a symmetric matrix, ensuring stable dynamics. Inspired by graph theory, the authors proposed HNNs with arbitrary weight matrices, which, besides the biological plausibility, may lead to cycles in the state space \cite{DBLP:conf/ijcnn/GarimellaKG15}. As a particular case, Goles \textit{et. al} studied HNNs with antisymmetric weight matrices operating in parallel mode and, in this case, observed cycles of length four \cite{DBLP:journals/dam/Goles86}.

Several models of complex-valued Hopfield neural networks (CvHNNs) have been proposed in the literature \cite{castro18cnmac,DBLP:journals/tnn/JankowskiLZ96,kobayashi17e,lee06,Kobayashi2025Complex-valuedConnections}. We studied the dynamics of one such model with Hermitian weight matrices in \cite{murthy2004complex}. As an immediate follow-up to our previous study, this paper addresses structured CvHNNs, that is, Hopfield neural networks with synaptic weight matrices with specific forms. Precisely, we investigate CvHNN with skew Hermitian weight matrices. We also introduce two new matrix forms, namely braided Hermitian and braided skew-Hermitian, along with experiments and theoretical results regarding the dynamics of CvHNN with these forms of the weight matrix. The present paper is an extension of the conference paper \cite{Garimella_2016}.

\new{
The paper is organized as follows: The next section introduces the notation used throughout the document. Section \ref{sec:CvHNNs} discusses the dynamics of structured complex-valued Hopfield neural networks (HNNs), including models that use Hermitian and skew-Hermitian synaptic weight matrices. This section also explores the dynamics of complex-valued Hopfield neural networks (CvHNNs) with two new types of complex-valued weight matrices: braided Hermitian and braided skew-Hermitian matrices. In Section \ref{sec:experiments}, we present extensive computational experiments that investigate the dynamics of structured CvHNNs with various weight matrices. Section \ref{sec:polar} examines the dynamics of CvHNNs based on the polar representation of the synaptic weights. The paper concludes with some final remarks in Section \ref{sec:concluding}, which include future work and potential applications.
}  

\section{\new{Notation}} \label{sec:notation}

\new{
Before proceeding with the theoretical development, let us introduce the notations used throughout the paper:}
\begin{itemize}
    \item \new{\( i \) — The imaginary unit, where \( i^2 = -1 \).}
    \item \new{\( z = a + bi \) — A complex number with real part \( a \) and imaginary part \( b \).}
    \item \new{\( \text{Re}\{z\} \) and \( \text{Im}\{z\} \) — The real and imaginary parts of \( z \), respectively.}
    \item \new{\( \mathbb{C}^{N \times N} \) — The space of \( N \times N \) complex-valued matrices.}
    \item \new{\( M \) — The synaptic weight matrix of the complex-valued Hopfield neural network.}
    \item \new{\( T \) — The threshold vector of the neural network.}
    \item \new{$R = (M,T)$ -- Short notation for a CvHNN with synaptic weight $M$ and threshold vector $T$.}
    \item \new{\( S(t) \) — The state vector of the neural network at time \( t \).}
    \item \new{\( \text{sign}(z) \) — The complex-valued activation function, applied separately to real and imaginary parts.}
    \item \new{\( L \) — Cycle length, the number of steps after which the neural network returns to a previous state.}
    \item \new{\textbf{Hermitian matrix:} A matrix \( M \) satisfying \( M^* = M \).}
    \item \new{\textbf{Skew-Hermitian matrix:} A matrix \( M \) satisfying \( M^* = -M \).}
    \item \new{\textbf{Braided Hermitian matrix:} A matrix satisfying \( M = A + A^T i \) for some real matrix \( A \).}
    \item \new{\textbf{Braided Skew-Hermitian matrix:} A matrix satisfying \( M = A - A^T i \) for some real matrix \( A \).}
\end{itemize}

\section{Complex-Valued Hopfield Neural Networks} \label{sec:CvHNNs}

This section briefly reviews CvHNNs, followed by a study on models with a skew-Hermitian weight matrix. Theorems regarding the convergence of both models are presented.

As far as we know, the first complex-valued Hopfield neural networks were investigated by Noest in the late 1980s \cite{noest88a,noest88b}. In contrast to the previous works based on the phase of a complex number, one of us proposed CvHNN with the split sign function as a natural extension of the bipolar HNN to the complex domain \cite{murthy2004complex}. Despite their simplicity, split activation functions provide universal approximation capabilities for complex-valued and high-dimensional hypercomplex-valued neural networks, such as those based on quaternions or Clifford algebras \cite{arena1997nn,arena1998book,Buchholz2008OnPerceptrons,valle2024nn}. In this work, we will focus on CvHNNs that utilize the split-sign activation function.

Let us denote a complex number by $z = a+b\ii$, where $\ii$ denotes the imaginary unit $\ii^2 = -1$, $a$ is the real part, and $b$ is the imaginary part. The functions $\re{\cdot}$ and $\im{\cdot}$ yield respectively the real part and the imaginary part of its argument, that is, $\re{z}=a$ and $\im{z} = b$ for $z= a+b\ii$.

As the name suggests, synaptic weights, thresholds, and the neuron states of complex-valued Hopfield neural networks are complex numbers. Analogous to a regular (real-valued) HNN, a complex-valued model can be viewed as a weighted directed graph with neurons represented by the nodes and synaptic weights of connections associated with the respective edges. 

Consider a network with $N$ neurons. Let us arrange the synaptic weights $M_{ij}$'s in a matrix $M \in \mathbb{C}^{N\times N}$ and the neuron thresholds $T_i$'s in a vector $T \in \mathbb{C}^{N}$. For simplicity, we express the CvHNN as $R=(M,T)$. The state of the network at the time step $t$ is given by a vector $S(t)$, where the states of a neuron belongs to the set \begin{equation} \label{eq:split-states}
\mathcal{S} = \lbrace +1+\ii, +1-\ii, -1+\ii, -1-\ii \rbrace.
\end{equation}
The state of the $i$th neuron at time $t$ is given by the $i$th entry of the vector $S(t)$, which is represented by $S_i(t)$ for simplicity. 

The update rule for complex neurons is similar to that of real neurons. Formally, inspired by the update scheme by Bruck and Goodman \cite{DBLP:journals/tit/BruckG88}, the $i$th neuron is updated as follows
\begin{equation}
    \label{eq:update}
    S_i(t+1) = \sign \left(\sum_{j=1}^N M_{ij}S_j(t) - T_i \right), \quad \forall i=1,\ldots,N.
\end{equation}
The split-sign complex-valued function $\sign:\mathbb{C} \to \mathcal{S}$ is obtained by applying separately the real-valued sign function at the real and imaginary parts of its argument, that is, 
\begin{equation} \label{eq:split-sign}
\sign(a+b \ii) = \sign_{\mathbb{R}}(a) + \sign_{\mathbb{R}}(b) \ii,
\end{equation}
where $\sign_{\mathbb{R}}:\mathbb{R} \to \{-1,+1\}$ is the real-valued sign function defined by
\begin{equation}
 \sign_{\mathbb{R}}(x) = \begin{cases}
   +1, & x \geq 0, \\
   -1, & x<0.
 \end{cases}
\end{equation}
The update rule given by \eqref{eq:update} is usually applied in one of two manners based on the current state of the network:
\begin{itemize}
    \item \textbf{Serial mode:} At each step a single neuron in updated;
    \item \textbf{(Fully) Parallel mode:} All neurons are updated simultaneously.
\end{itemize}
Because all neurons are updated simultaneously in the parallel mode, it is also called the \textbf{synchronous} update. By contrast, the serial update mode is also known as \textbf{asynchronous} update.

We say that a network $R=(M,T)$ stabilizes at a \textbf{cycle of length $L$}, where $L$ is a positive integer if there exists $t_0 \geq 0$ such that $S(t+L)=S(t)$ for all $t \geq t_0$, that is, the network returns to the same state after $L$ steps. Moreover, we say that the network $R$ converges to a \textbf{stable state} if it stabilizes at a cycle of length $L=1$. In other words, the network converges to a stable state if all the neurons remain in their state, that is, there exists $t_0\geq0$ such that $S(t+1)=S(t)$ for all $t \geq t_0$. 

The update mode plays a key role in the dynamics of Hopfield neural networks \cite{DBLP:journals/tit/BruckG88,Castro2018SomeNetworks,hopfield1982neural}. In fact, a network $R=(M, T)$ operating in serial mode may settle at a stable state, but the same network $R$ may not stabilize in parallel update mode.

One usually presents an energy function to analyze the dynamics of a recurrent neural network. Briefly, an energy function is a bounded real-valued mapping on the set of all network states. Moreover, it must decrease when evaluated on consecutive but different network states. The energy function often depends on the operation mode of the network \cite{murthynovel}. Inspired by the real-valued HNN  \cite{hopfield1982neural,Goles-Chacc1985DecreasingNetworks,DBLP:journals/tit/BruckG88},
we consider the following expressions as energy functions in this paper:
\begin{itemize}
    \item \textbf{Serial mode:}
    \begin{equation}
        \label{eq:energy_serial}
        E_S\big(S\big) = -\re{S^* MS - 2S^*T}.
    \end{equation}
    \item \textbf{Parallel mode:}
    \begin{equation}
        \label{eq:energy_parallel}
        E_P\big(S_1,S_2\big) = -\re{S_1^* M S_2 - (S_1+S_2)^*T}.
    \end{equation}
\end{itemize}
where $S,S_1,S_2 \in \mathbb{S}^N$ denote arbitrary states of the network and $S^*$ denotes the conjugate (Hermitian) transpose of $S$, that is, $S^* = \bar{S}^T$. In analyzing the dynamics of CvHNNs, one typically considers \( S = S(t) \) for the serial update mode, and uses \( S_1 = S(t) \) and \( S_2 = S(t-1) \) for the parallel update mode, where \( S(t) \) and \( S(t-1) \) represent consecutive states of the network.

We would like to remark that both \( E_S \) and \( E_P \) are bounded real-valued functions because the set of states in the complex-valued Hopfield neural network (HNN) is discrete. Additionally, the function \( E_P \) depends on two network states. However, \( E_P \) and \( E_S \) are equal when they are calculated for the same network state; specifically, \( E_P(S, S) = E_S(S) \) when \( S_1 = S_2 = S \). 

From \eqref{eq:energy_serial} and \eqref{eq:energy_parallel} it is clear that the form of the synaptic matrix directly impacts the calculation of the energy function. Therefore $M$ plays a key role in convergence theorems. The following addresses some particular forms for the complex-valued matrix $M$, including the Hermitian and skew-Hermitian. Two novel classes of synaptic weight matrices are introduced subsequently.

\subsection{CvHNNs with Hermitian weight matrix}

In the ordinary CvHNN, the synaptic weight matrix $M$ is Hermitian, i.e., $M^* = M$. Recall that the main diagonal elements are pure real numbers (i.e., complex numbers with null imaginary parts) in a Hermitian matrix. The following theorem addresses the dynamics of ordinary CvHNNs using either serial or parallel updates.
\begin{Theorem}
\label{thm:Hermitian}
 Consider a complex-valued Hopfield neural network $R=(M,T)$ where $M \in \mathbb{C}^{N \times N}$ is Hermitian. 
 \begin{enumerate}
     \item If $R$ is operated in serial mode and the main diagonal elements of $M$ are non-negative, the network always settles at a stable state.
     \item If $R$ is operated in parallel mode, it reaches a cycle of length $L \leq 2$, i.e., it either converges to a stable state or stabilizes at a cycle of length $2$.
 \end{enumerate}
\end{Theorem}



\new{
The above result for serial mode operation uses the energy defined in \eqref{eq:energy_serial} and has been proved in \cite{murthy2008multidimensional}. The parallel case involves transforming the CvHNN into a real-valued model with a symmetric weight matrix, similar to the approach we will use in the proof of Theorem \ref{thm:braided}. We then apply the results for the real-valued case as detailed in \cite{DBLP:journals/tit/BruckG88} or \cite{Goles-Chacc1985DecreasingNetworks}.}

\subsection{CvHNN with skew-Hermitian weight matrix}

Consider a structured CvHNN with a skew-Hermitian matrix $M \in \mathbb{C}^{N \times N}$, the complex-valued analogous to the antisymmetric real matrix concept. Recall that $M\in \mathbb{C}^{N \times N}$ is a skew-Hermitian matrix if $M^* = -M$. As a consequence, the elements in the main diagonal of $M$ must be pure imaginary numbers. 

In the real-valued case, such configuration leads to a cycle of length $4$ when operated in parallel mode \cite{DBLP:journals/dam/Goles86}. The following theorem, whose proof was firstly outlined in \cite{Garimella_2016}, provides an equivalent result for the complex-valued structured HNN.


\begin{Theorem} \label{thm:Skew-Hermitian}
Let $R=(M,0)$ be a complex-valued Hopfield neural network where $M$ is a skew-Hermitian matrix, and there are no thresholds, that is, $T_i=0$ for all $i=1,\ldots,N$. For $R$ operating in parallel mode, the network stabilizes in a cycle of length $L = 4$. 
\end{Theorem}

\begin{proof}
We prove by converting the CvHNN into a real-valued model and, then, using the results from \cite{DBLP:journals/dam/Goles86}. By hypothesis, the threshold vector is the null vector, i.e., $T=0$. The weight matrix and the network's state vector can be written respectively as follows:
\begin{equation*}
    M = A + B\ii \quad \mbox{and} \quad S(t)  = X(t) + Y(t) \ii. 
\end{equation*}
From \eqref{eq:update} and \eqref{eq:split-sign}, the real and the imaginary parts of the $i$th neuron are updated as follows, respectively:
\begin{equation*}
    X_i(t+1) = \sign_{\mathbb{R}} \left(
    \sum_{j=1}^N \left(A_{ij} X_j(t) - B_{ij} Y_j(t)\right) \right),
\end{equation*}
and
\begin{equation*}
    Y_i(t+1) = \sign_{\mathbb{R}} \left(
    \sum_{j=1}^N \left(B_{ij} X_j(t) + A_{ij} Y_j(t)\right) \right).
\end{equation*}
Alternatively, using a matrix notation we have
\begin{equation*}
    \begin{bmatrix}
    X(t+1) \\ Y(t+1)
    \end{bmatrix} = 
    \sign_{\mathbb{R}} \left( \begin{bmatrix}
    A & -B \\ B & A
    \end{bmatrix} \begin{bmatrix}
    X(t) \\ Y(t)
    \end{bmatrix} \right),
\end{equation*}
where $\sign_{\mathbb{R}}:\mathbb{R} \to \{-1,+1\}$ is evaluated in a component-wise manner. Therefore, the CvHNN $R=(M,0)$ operating in parallel mode corresponds to a synchronous real-valued HNN with weight matrix 
\begin{equation*}
    W = \begin{bmatrix}
    A & -B \\ B & A
    \end{bmatrix}.
\end{equation*}
Now, if $M = A + B\ii$ is skew-Hermitian, then $A^T = -A$ and $B^T = B$. As a consequence, the weight matrix $W$ of the real-valued HNN satisfies
\begin{equation*}
    W^T = \begin{bmatrix}
    A & -B \\ B & A
    \end{bmatrix}^T
    = \begin{bmatrix}
    A^T & B^T \\ -B^T & A^T
    \end{bmatrix}
    = \begin{bmatrix}
    -A & B \\ -B & -A
    \end{bmatrix}
    = -\begin{bmatrix}
    A & -B \\ B & A
    \end{bmatrix} = - W.
\end{equation*}
In other words, $W$ is anti-symmetric. According to Goles \cite{DBLP:journals/dam/Goles86}, the synchronous real-valued HNN and, equivalently, the CvHNN $R=(M,0)$ operating in parallel, stabilizes at a cycle of length $L=4$.
\end{proof}

\begin{Example} \label{ex:skew-Hermitian}
A synchronous CvHNN with a skew-Hermitian matrix may stabilize in cycles of different lengths if the threshold is not the null vector. Figure \ref{fig:skew-Hermitian} illustrates this remark by showing the cycle length probability of synchronous CvHNN with and without threshold vector. Precisely, we generated 40,000 CvHNN models with skew-Hermitian weight matrices and estimated the cycle length by probing them with a random input vector. The number of neurons ranged from 5 to 70. The real and imaginary parts of the weights have been drawn from a uniform distribution on $[-1,1]$ while the non-null thresholds are uniformly distributed on the interval $[-N,N]$. Note from Figure \ref{fig:skew-Hermitian} that the CvHNN with arbitrary thresholds stabilizes at cycles of various lengths in addition to the cycles of length 4.
\begin{figure}[t]
    \centering
    \begin{tabular}{cc}
    a) $T=0$ & b) $T$ arbitrary \\
    \includegraphics[width=0.45\columnwidth]{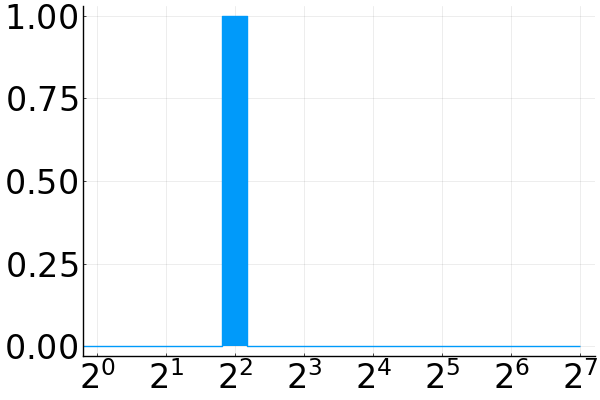} &
    \includegraphics[width=0.45\columnwidth]{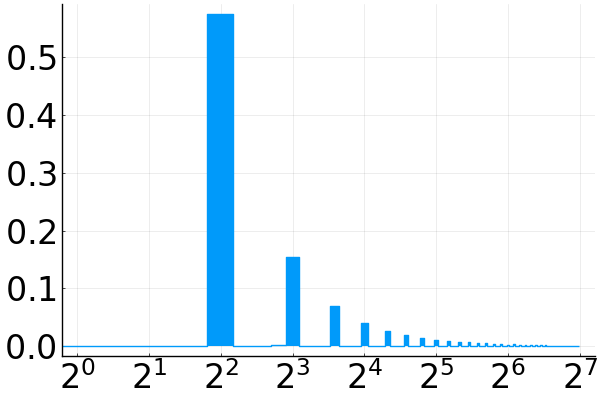} \end{tabular}
    \caption{Probability of cycle length in synchronous CvHNNs with a skew-Hermitian weight matrix and either null thresholds ($T=0$) or random arbitrary thresholds ($T$).}
    \label{fig:skew-Hermitian}
\end{figure}
\end{Example}

\subsection{CvHNNs with braided weight matrices} \label{sec:4}

Let us begin by reviewing two types of complex matrices that were introduced in our conference paper \cite{Garimella_2016}. After that, we present a theorem that addresses the dynamics of the CvHNN when the synaptic matrix assumes either of these forms.

\begin{Definition}[Braided Hermitian Matrix] \label{def:braided}
A complex-valued matrix $M$ is said to be \textbf{braided Hermitian} if $M = A + A^T \ii$ for some square real matrix $A$.
\end{Definition} 

For example, consider the real-valued matrix
\begin{equation}
\label{eq:matrix-A}
    A = \begin{bmatrix}
    -3 & -2 & \phantom{-}1 & \phantom{-}5 \\
     \phantom{-}4 & \phantom{-}6 & -1 & -4 \\
     \phantom{-}8 & -2 & \phantom{-}7 & 10 \\
    -7 & \phantom{-}6 & \phantom{-}2 & \phantom{-}5 \\
    \end{bmatrix}.
\end{equation}
The associated braided Hermitian matrix is
\begin{equation*}
    M = \begin{bmatrix}
     -3-3\ii &-2+4\ii \phantom{-}&1+8\ii &\phantom{-}5-7\ii \\
 \phantom{-}4-2\ii &\phantom{-}6+6\ii &-1-2\ii &-4+6\ii \\
  \phantom{-}8+1\ii &-2-1\ii &\phantom{-}7+7\ii &10+2\ii \\
 -7+5\ii &\phantom{-}6-4\ii &2+10\ii &\phantom{-}5+5\ii \\
    \end{bmatrix}.
\end{equation*}
From Definition \ref{def:braided}, we note that a braided Hermitian matrix is such that
$$ M^T = \ii \overline{M}. $$

\begin{Definition}[Braided Skew-Hermitian Matrix] \label{def:skew-braided}
A complex-valued matrix $M$ is \textbf{braided skew-Hermitian} if it satisfies $M= A - A^T \ii$ for some square real matrix $A$. 
\end{Definition}

For example, the braided skew-Hermitian matrix derived from the real-valued $A$ given by \eqref{eq:matrix-A} is
\begin{equation*}
   M = \begin{bmatrix}
 -3+3\ii &-2-4\ii &\phantom{-}1-8\ii &\phantom{-}5+7\ii \\
 \phantom{-}4+2\ii &\phantom{-}6-6\ii &-1+2\ii &-4-6\ii \\
  \phantom{-}8-1\ii &-2+1\ii &\phantom{-}7-7\ii &10-2\ii \\
 -7-5\ii & \phantom{-}6+4\ii \phantom{-}&2-10\ii &5-5\ii \\
\end{bmatrix}.
\end{equation*}
Note that braided skew-Hermitian matrices have the property:
$$ M^T = - \ii \overline{M}. $$

With these two definitions in hand, we have the following theorem concerning the dynamics of complex-valued Hopfield neural networks:
\begin{Theorem} 
\label{thm:braided}
Let $R=(M,0)$ be a CvHNN with no thresholds, i.e., $T=0$. If $M$ is braided Hermitian or braided skew-Hermitian, the network converges to a cycle of length $8$ when operated in parallel mode.
\end{Theorem}
The proof Theorem \ref{thm:braided} can be found in our conference paper \cite{Garimella_2016}. Like the synchronous CvHNN with a skew-Hermitian matrix, the CvHNN $R=(M,T)$ with either a braided Hermitian or a braided skew-Hermitian weight matrix may stabilize in cycles of different lengths if the threshold is not the null vector. The following example illustrates this remark.

\begin{Example} \label{ex:braided}
Figure \ref{fig:braided} shows the histograms of the synchronous CvHNN cycle lengths with braided Hermitian and braided skew-Hermitian matrices, both with and without threshold vector. Precisely, we generated 40,000 CvHNN models with braided Hermitian and skew-Hermitian weight matrices and estimated the cycle length by probing them with a random input vector. The number of neurons ranged from 5 to 70. The real and imaginary parts of the weights have been drawn from a uniform distribution on $[-1,1]$ while the non-null thresholds are uniformly distributed on the interval $[-N,N]$. \new{In Figure \ref{fig:braided}, the x-axis represents the cycle length, while the y-axis indicates the probability.} Note that the CvHNN models with arbitrary thresholds stabilize at cycles of many different lengths besides the cycles of length $L=8$.
\begin{figure}[t]
    \centering
    \begin{tabular}{cc}
    \multicolumn{2}{c}{\underline{Braided Hermitian}} \\
    a) $T=0$ & b) $T$ arbitrary \\
    \includegraphics[width=0.45\columnwidth]{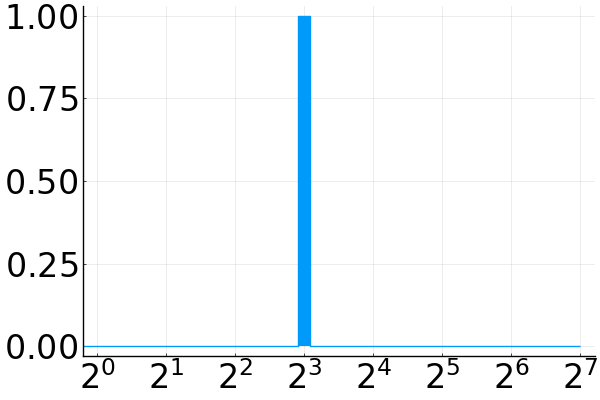} &
    \includegraphics[width=0.45\columnwidth]{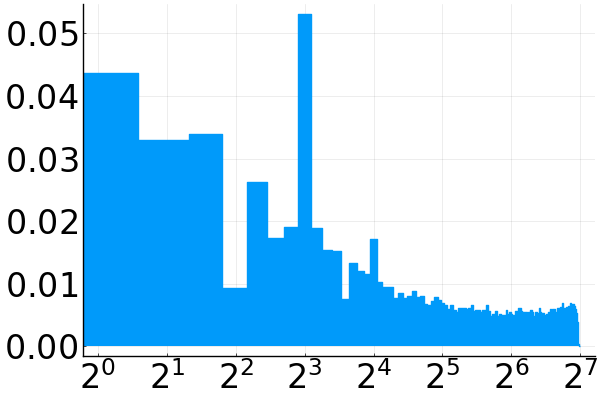} \\
    \multicolumn{2}{c}{\underline{Braided Skew-Hermitian}} \\
    c) $T=0$ & d) $T$ arbitrary \\
    \includegraphics[width=0.45\columnwidth]{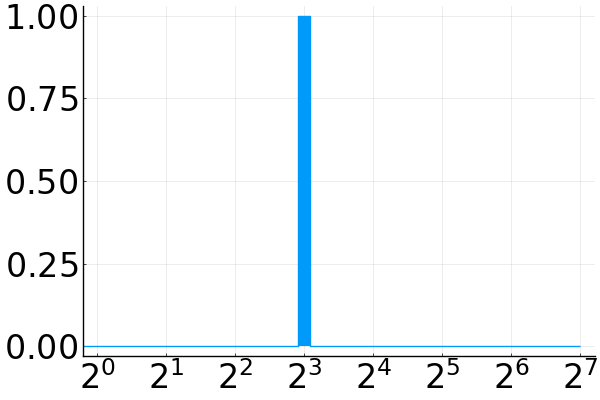} &
    \includegraphics[width=0.45\columnwidth]{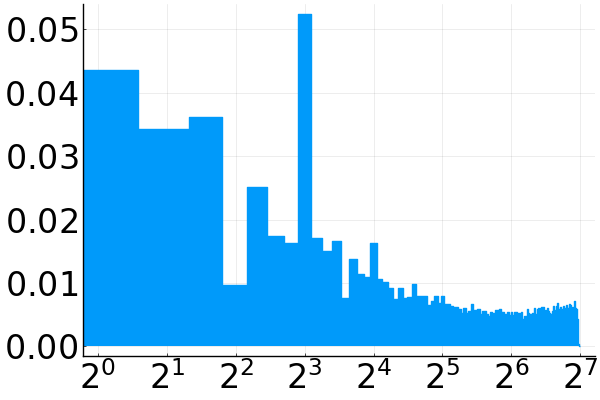}
    \end{tabular}
    \caption{\new{Probability of cycle length in synchronous CvHNNs with a braided Hermitian and braided skew-Hermitian weight matrix, considering both null thresholds ($T=0$) and random thresholds ($T$ arbitrary).}}
    \label{fig:braided}
\end{figure}
\end{Example}

\section{Computational Experiments with Synchronous Complex-valued Hopfield Neural Networks} \label{sec:experiments}

This section presents the results of computational simulations using various structured CvHNN models operating synchronously, that is, using parallel updates. The primary objective of these experiments is to assess the existence of generalized convergence results for specific complex-valued synaptic weight matrices. Because the threshold vector may alter the cycle length (cf. Example \ref{ex:skew-Hermitian}), we refrain from including it. In other words, we consider $T=0$ in the following computational experiments.

The experiments are conducted based on the structure of the weight matrix \( M \) in a complex-valued Hopfield neural network model, denoted as \( R = (M, 0) \). To explore the dynamics of these CvHNN models, we represent the weight matrix \( M \) as \( M = A + B \ii \), where \( A \) and \( B \) are real matrices. In our discussion, we will examine the results based on the characteristics of \( A \) and \( B \). Recall that we presented theorems related to two forms of \( M \) in Section \ref{sec:CvHNNs}. Namely:
\begin{enumerate}
    \item $M$ Hermitian (Theorem \ref{thm:Hermitian}), which implies $A$ symmetric and $B$ antisymmetric;
    \item $M$ skew-Hermitian (Theorem \ref{thm:Skew-Hermitian}), which implies $A$ antisymmetric and $B$ symmetric.
    \item $M$ braided Hermitian, which implies $A$ arbitrary and $B=A^T$.
    \item $M$ skew-braided Hermitian, which implies $A$ arbitrary and $B=-A^T$.
\end{enumerate}
In this section, we explore several additional configurations. The real-valued matrices \( A \) and \( B \) can take one of three forms:
\begin{enumerate}
    \item \textbf{Symmetric:} In this case, \( A^T = A \) or \( B^T = B \).
    \item \textbf{Antisymmetric:} Here, \( A^T = -A \) or \( B^T = -B \).
    \item \textbf{Arbitrary}: In this scenario, there are no constraints imposed on \( A \) or \( B \).
\end{enumerate}
This results in a total of nine possible combinations, excluding the cases addressed previously.

As in Examples \ref{ex:skew-Hermitian} and \ref{ex:braided}, we randomly generate 40,000 instances for each configuration. Each instance consists of a pair of matrices $A$ and $B$ and an initial state $S(0)$. The entries of $A$ and $B$ and the components of the initial state $S(0)$ have all been drawn from a uniform distribution. Precisely, the matrices $A$ and $B$ have been obtained according to one of the following cases:
\begin{itemize}
    \item \textbf{Positive}, when the entries in the upper triangular part belong to the interval $[0,1]$.
    \item \textbf{Negative}, when the entries in the upper triangular part belong to the interval $[-1,0]$.
    \item \textbf{Arbitrary}, in which the entries in the upper triangular part belong to $[-1,1]$.
\end{itemize}
Furthermore, the number of neurons $5\leq N \leq 70$ is randomly chosen in each instance. Additional conditions regarding the matrices are detailed below. In the following computational experiments, the CvHNN models are all operated in parallel mode, and the outcome is classified according to the length of the $L$ of the observed cycles. \new{Each configuration is accompanied by a brief discussion, while a summary of the results is provided in Subsection \ref{ssec:summary}.}

\subsection{A is symmetric and B is symmetric}

Let us begin by considering that $A$ and $B$ are symmetric matrices. In this case, the complex-valued weight matrix $M$ is also symmetric. 
Figure \ref{fig:SymSym} shows the histograms of the cycle length of the resulting CvHNN models. This figure also shows the probability of the most probable cycle length.

\begin{figure}[t]
    \begin{tabular}{ll|ccc}
    && \multicolumn{3}{c}{\underline{Symmetric matrix $A$}} \\ 
    && \small{Positive} & \small{Negative} & \small{Arbitrary} \\ \hline
    && \small{a) $\Pr[L=8]=0.95$} 
    & \small{b) $\Pr[L=8]=0.94$} 
    & \small{c) $\Pr[L=4]=0.98$} \\
    \multirow{3}{*}{\hfill \rotatebox[origin=lB]{90}{\underline{Symmetric matrix $B$}}} & \rotatebox{90}{$\quad$ \small{Positive}} & 
    \includegraphics[width=0.26\columnwidth]{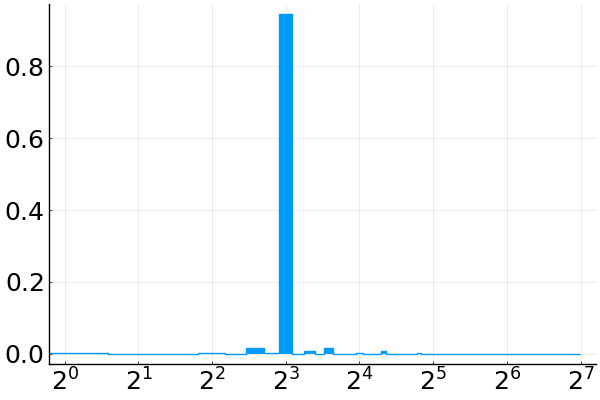} &
    \includegraphics[width=0.26\columnwidth]{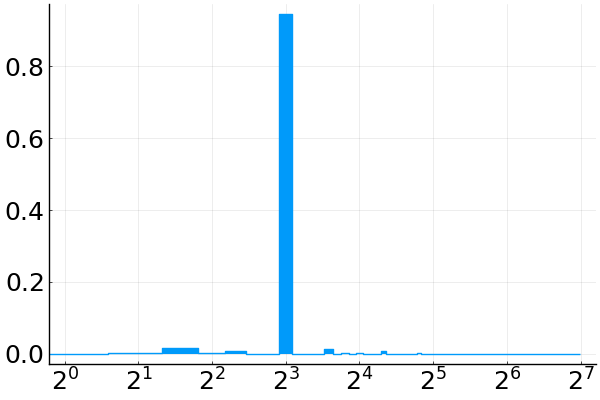} &
    \includegraphics[width=0.26\columnwidth]{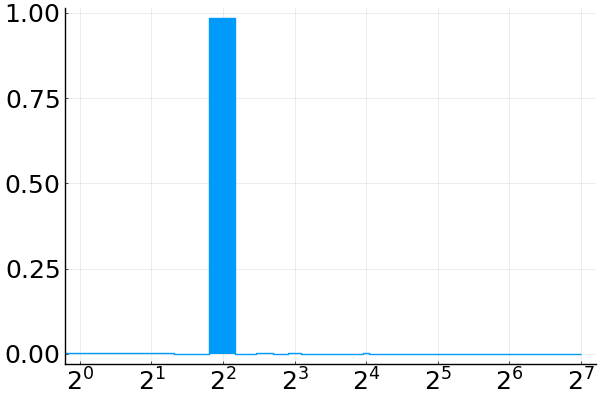} \\
    && \small{d) $\Pr[L=8]=0.95$} 
    & \small{e) $\Pr[L=8]=0.95$} 
    & \small{f) $\Pr[L=4]=0.98$} \\
    & \rotatebox{90}{$\quad$ \small{Negative}}  & 
    \includegraphics[width=0.26\columnwidth]{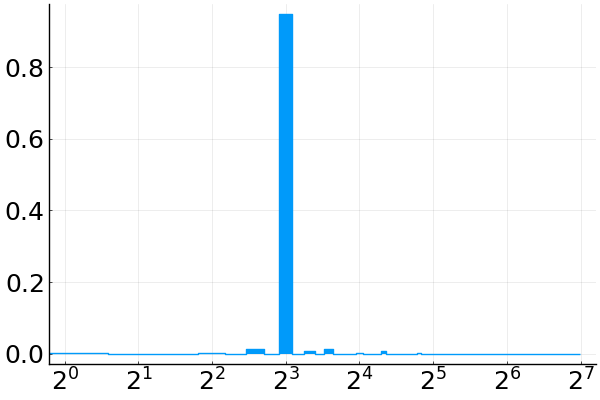} &
    \includegraphics[width=0.26\columnwidth]{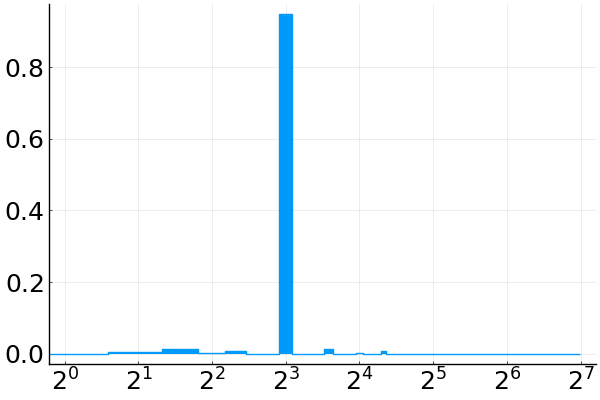} &
    \includegraphics[width=0.26\columnwidth]{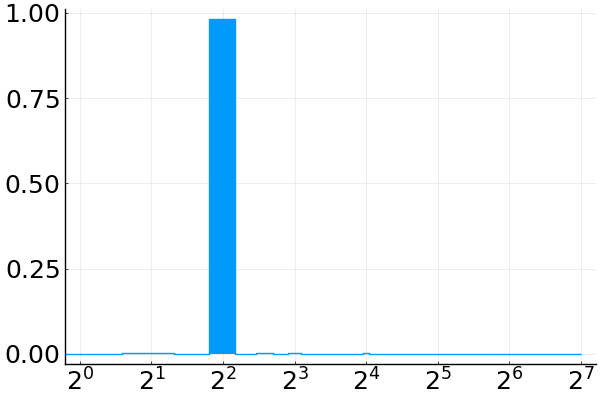} \\
    && \small{g) $\Pr[L=1]=0.98$} & 
    \small{h) $\Pr[L=2]=0.98$} & 
    \small{i) $\Pr[L=8]=0.02$} \\
    & \rotatebox{90}{$\quad$ \small{Arbitrary}}  & 
    \includegraphics[width=0.26\columnwidth]{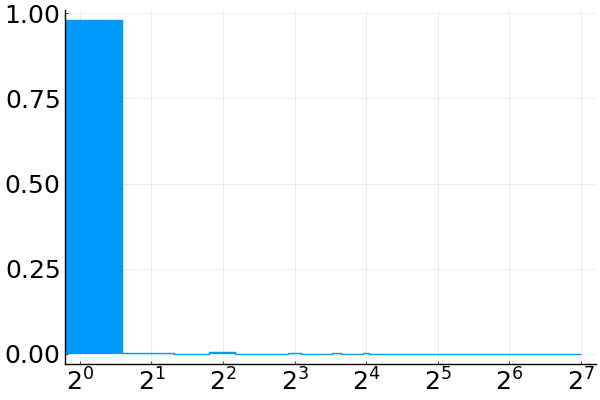} &
    \includegraphics[width=0.26\columnwidth]{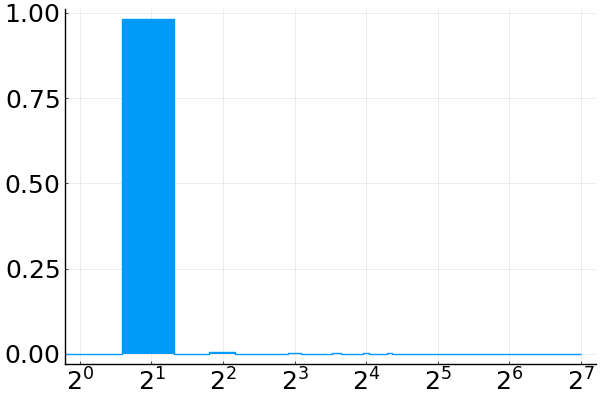} &
    \includegraphics[width=0.26\columnwidth]{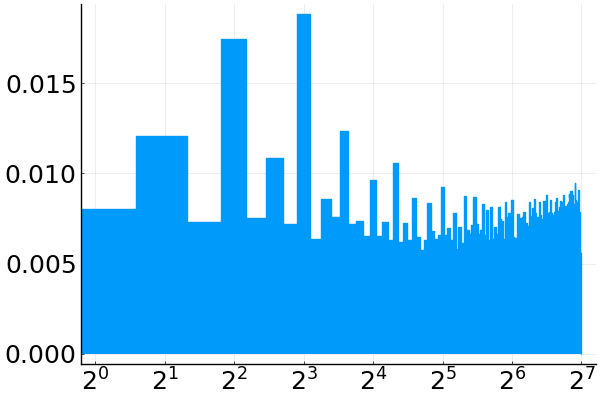}
    \end{tabular}
    \caption{\new{Probability of cycle length in synchronous structured CvHNNs when both matrices A and B are symmetric.}}
    \label{fig:SymSym}
\end{figure}

Note that the CvHNN has a high probability of stabilizing in a cycle of length $L=8$ when the entries of $A$ and $B$ have non-arbitrary signs. On the one hand, if the entries of $A$ have arbitrary signs but $B$ is either positive or negative, the CvHNN has a high probability of stabilizing in a cycle of length $L=4$. On the other hand, if $A$ is either positive or negative, but the entries of $B$ have arbitrary signs, the CvHNN has a high probability of stabilizing at a cycle of length less than or equal to $L=2$. Finally, the CvHNN exhibits cycles of several different lengths with similar probabilities when $A$ and $B$ have entries with arbitrary signs. 

\subsection{A is symmetric and B is arbitrary}

\begin{figure}[t]
    \begin{tabular}{ll|ccc}
    && \multicolumn{3}{c}{\underline{Symmetric matrix $A$}} \\ 
    && \small{Positive} & \small{Negative} & \small{Arbitrary} \\ \hline
    && \small{a) $\Pr[L=8]=0.96$} 
    & \small{b) $\Pr[L=8]=0.96$} 
    & \small{c) $\Pr[L=4]=0.98$} \\
    \multirow{3}{*}{\hfill \rotatebox[origin=lB]{90}{\underline{Arbitrary matrix $B$}}} & \rotatebox{90}{$\quad$ \small{Positive}} & 
    \includegraphics[width=0.26\columnwidth]{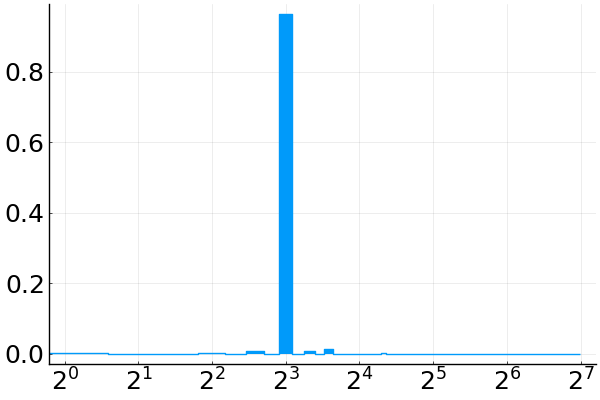} &
    \includegraphics[width=0.26\columnwidth]{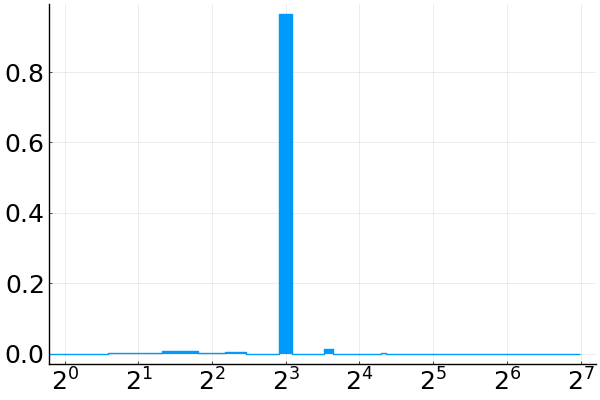} &
    \includegraphics[width=0.26\columnwidth]{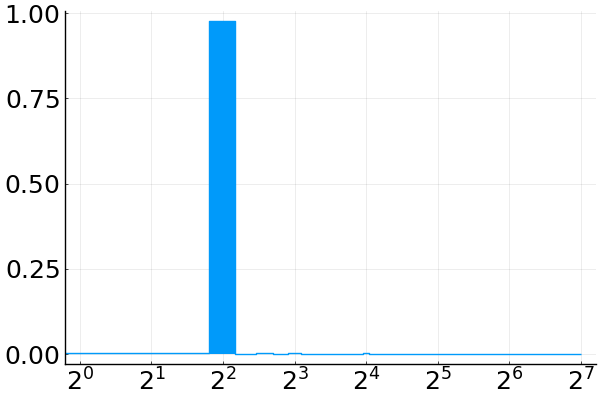} \\
    && \small{d) $\Pr[L=8]=0.96$} 
    & \small{e) $\Pr[L=8]=0.96$} 
    & \small{f) $\Pr[L=4]=0.98$} \\
    & \rotatebox{90}{$\quad$ \small{Negative}}  & 
    \includegraphics[width=0.26\columnwidth]{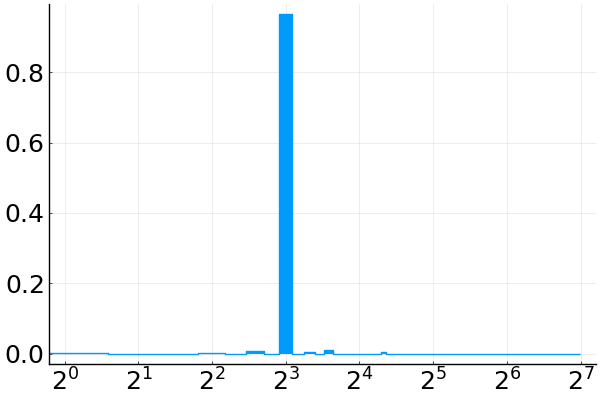} &
    \includegraphics[width=0.26\columnwidth]{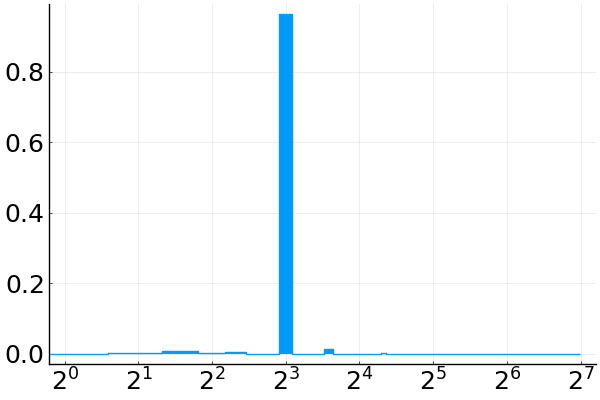} &
    \includegraphics[width=0.26\columnwidth]{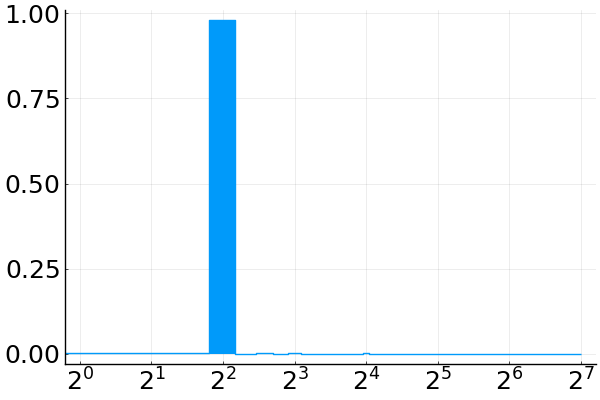} \\
    && \small{g) $\Pr[L=1]=0.99$} & 
    \small{h) $\Pr[L=2]=1.0$} & 
    \small{i) $\Pr[L=2]=0.8$} \\
    & \rotatebox{90}{$\quad$ \small{Arbitrary}}  & 
    \includegraphics[width=0.26\columnwidth]{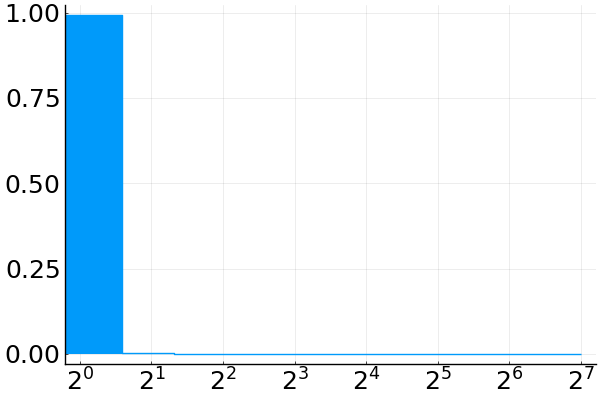} &
    \includegraphics[width=0.26\columnwidth]{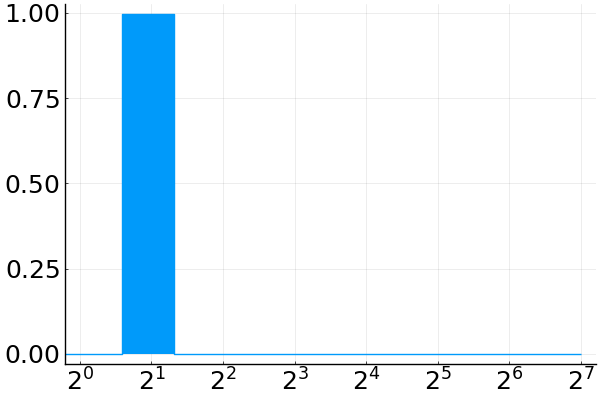} &
    \includegraphics[width=0.26\columnwidth]{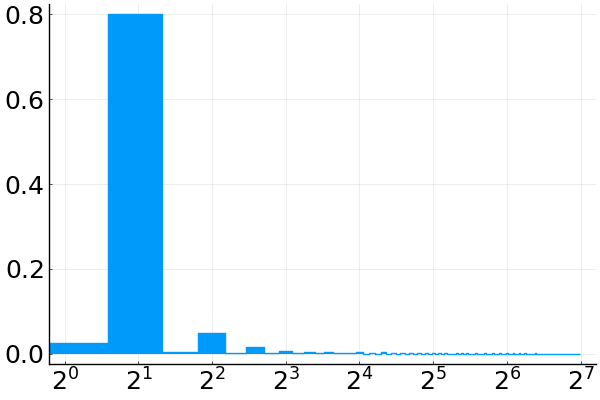}
    \end{tabular}
    \caption{\new{Probability of cycle length in synchronous structured CvHNNs when $A$ is symmetric and $B$ is an arbitrary matrix.}}
    \label{fig:SymArb}
\end{figure}

Let us now address the case where $A$ is symmetric but $B$ has no symmetry restrictions. Restrictions are still applied to the sign of elements in $B$. The histogram of the resulting cycle lengths is shown in Figure \ref{fig:SymArb}. The probability of the most frequent cycle length is also provided in Figure \ref{fig:SymArb}. 

A behavior similar to that of the previous case is observed except when the entries of $A$ and $B$ have no sign constraints. Although we observed a prevalence of cycles of length $L=2$, the CvHNN based on a symmetric matrix $A$ eventually stabilized in extremely large cycles.

\subsection{A is antisymmetric and B is antisymmetric}

\begin{figure}[t]
    \begin{tabular}{ll|ccc}
    && \multicolumn{3}{c}{\underline{Antisymmetric matrix $A$}} \\ 
    && \small{Positive} & \small{Negative} & \small{Arbitrary} \\ \hline
    && \small{a) $\Pr[L=8]=0.97$} 
    & \small{b) $\Pr[L=8]=0.98$} 
    & \small{c) $\Pr[L=2]=0.26$} \\
    \multirow{3}{*}{\hfill \rotatebox[origin=lB]{90}{\underline{Antisymmetric matrix $B$}}} & \rotatebox{90}{$\quad$ \small{Positive}} & 
    \includegraphics[width=0.26\columnwidth]{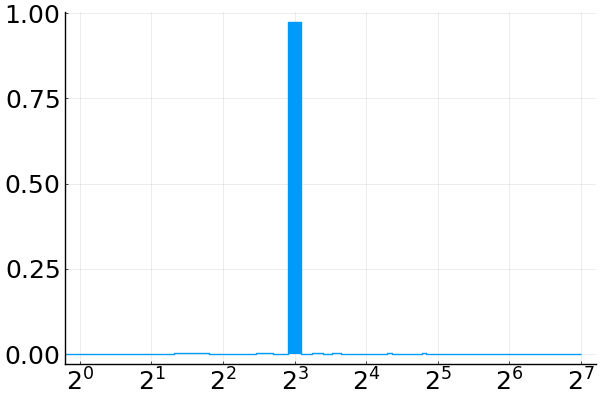} &
    \includegraphics[width=0.26\columnwidth]{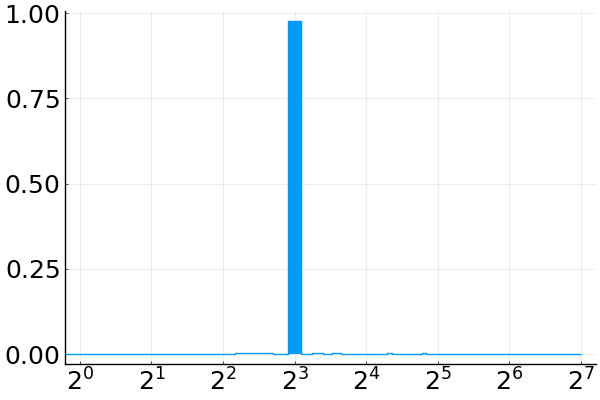} &
    \includegraphics[width=0.26\columnwidth]{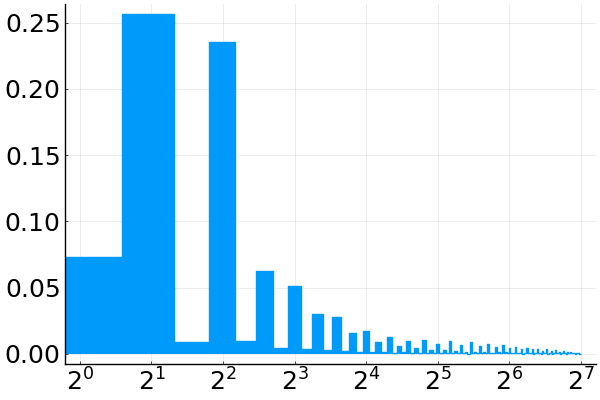} \\
    && \small{d) $\Pr[L=8]=0.98$} 
    & \small{e) $\Pr[L=8]=0.98$} 
    & \small{f) $\Pr[L=2]=0.26$} \\
    & \rotatebox{90}{$\quad$ \small{Negative}}  & 
    \includegraphics[width=0.26\columnwidth]{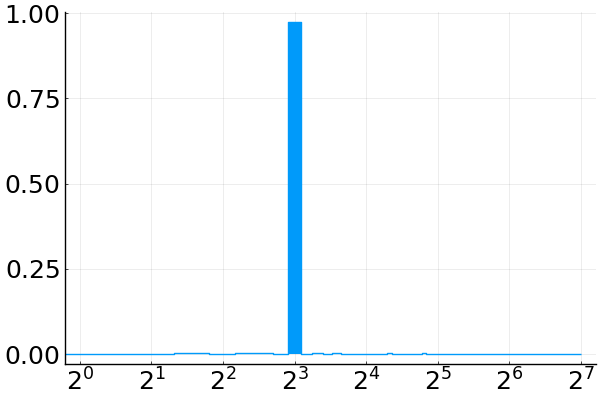} &
    \includegraphics[width=0.26\columnwidth]{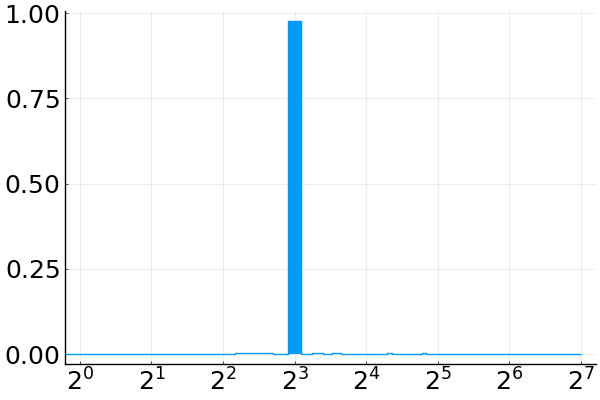} &
    \includegraphics[width=0.26\columnwidth]{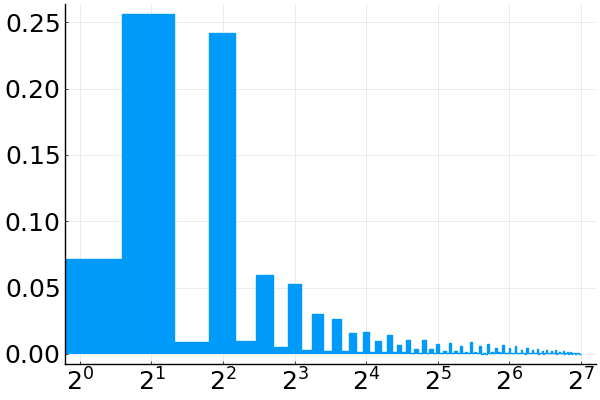} \\
    && \small{g) $\Pr[L=4]=0.56$} & 
    \small{h) $\Pr[L=4]=0.55$} & 
    \small{i) $\Pr[L=8]=0.02$} \\
    & \rotatebox{90}{$\quad$ \small{Arbitrary}}  & 
    \includegraphics[width=0.26\columnwidth]{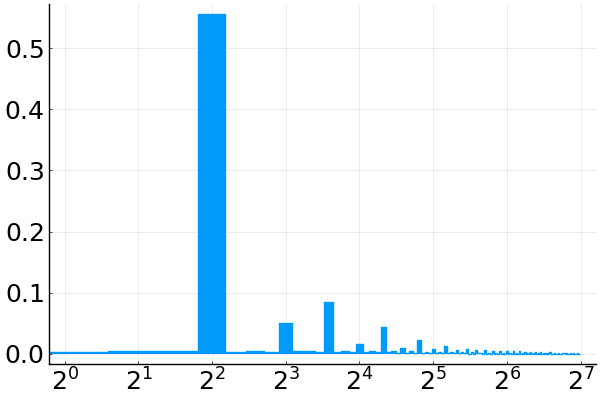} &
    \includegraphics[width=0.26\columnwidth]{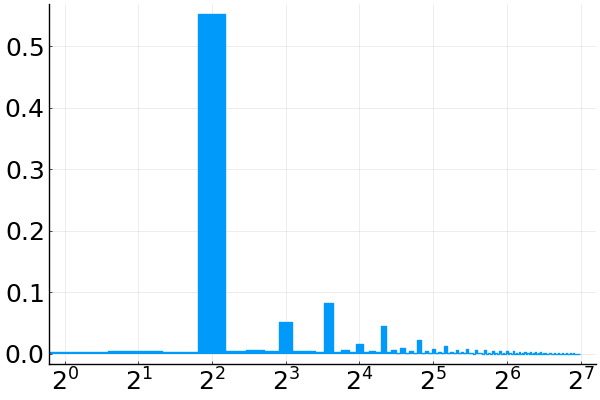} &
    \includegraphics[width=0.26\columnwidth]{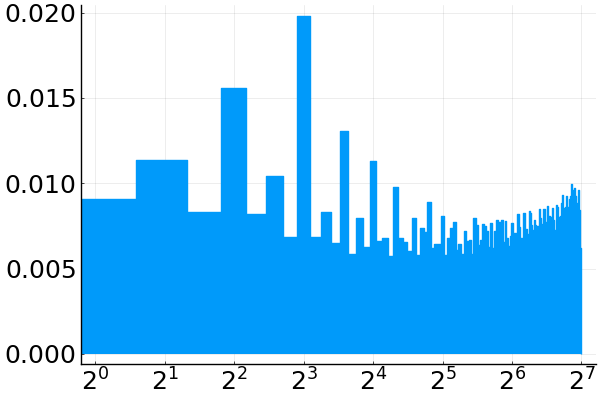}
    \end{tabular}
    \caption{\new{Probability of cycle length in synchronous structured CvHNNs when both $A$ and $B$ are antisymmetric matrices.}}
    \label{fig:ASymASym}
\end{figure}

When both $A$ and $B$ are antisymmetric, we have $M^T = -M$, i.e., $M$ is also antisymmetric. Figure \ref{fig:ASymASym} shows the histogram of the cycle lengths $L$ obtained when constraints are eventually imposed on the upper triangular part of the matrices $A$ and $B$. This figure also shows the probability of the most frequent cycle length for each sign constraint.

Analogous to the case in which $A$ and $B$ are symmetric matrices, the CvHNN has a high probability of stabilizing in a cycle of length $L=8$ when $A$ and $B$ are both anti-symmetric matrices with constraints in the sign of their entries. No dominant cycle length is observed when the entries of $A$ and $B$ have arbitrary signs. For example, if the entries of $A$ are positive and those of $B$ have arbitrary signs, the most probable cycle length is $L=4$, with probability $\Pr[L=4]=0.56$. However, the mean cycle length is $13.2$ with a standard deviation of $21.5$ when $A$ is positive and $B$ has arbitrary signs.

\subsection{A is antisymmetric and B is arbitrary} 

\begin{figure}[t]
    \begin{tabular}{ll|ccc}
    && \multicolumn{3}{c}{\underline{Antisymmetric matrix $A$}} \\ 
    && \small{Positive} & \small{Negative} & \small{Arbitrary} \\ \hline
    && \small{a) $\Pr[L=4]=1.00$} 
    & \small{b) $\Pr[L=4]=1.00$} 
    & \small{c) $\Pr[L=4]=1.00$} \\
    \multirow{3}{*}{\hfill \rotatebox[origin=lB]{90}{\underline{Arbitrary matrix $B$}}} & \rotatebox{90}{$\quad$ \small{Positive}} & 
    \includegraphics[width=0.26\columnwidth]{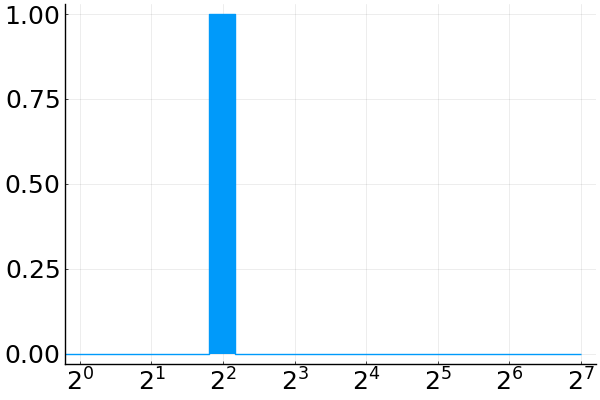} &
    \includegraphics[width=0.26\columnwidth]{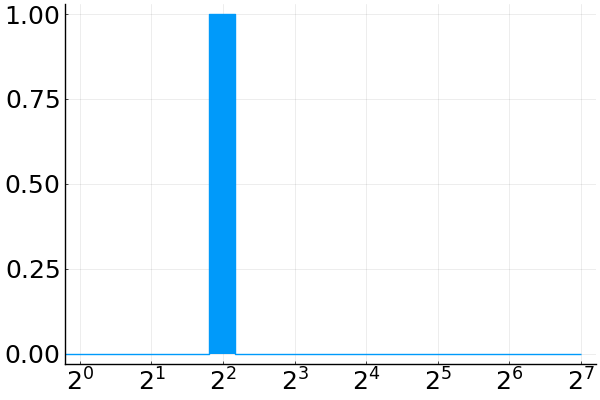} &
    \includegraphics[width=0.26\columnwidth]{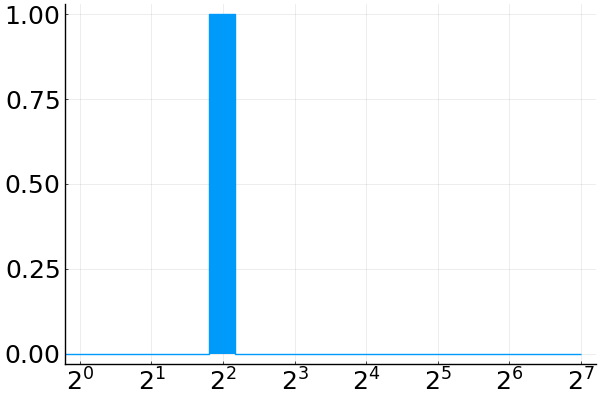} \\
    && \small{d) $\Pr[L=4]=1.00$} 
    & \small{e) $\Pr[L=4]=1.00$} 
    & \small{f) $\Pr[L=4]=1.00$} \\
    & \rotatebox{90}{$\quad$ \small{Negative}}  & 
    \includegraphics[width=0.26\columnwidth]{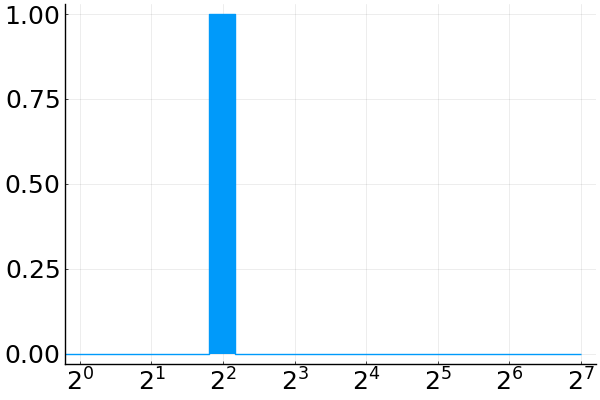} &
    \includegraphics[width=0.26\columnwidth]{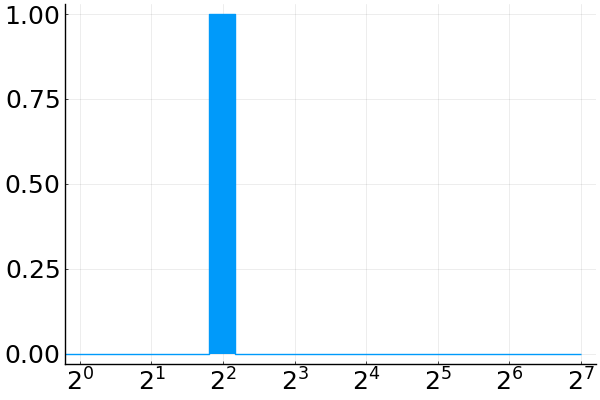} &
    \includegraphics[width=0.26\columnwidth]{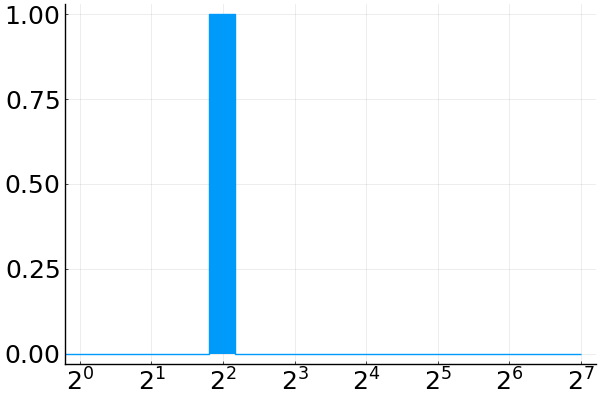} \\
    && \small{g) $\Pr[L=4]=0.92$} & 
    \small{h) $\Pr[L=4]=0.92$} & 
    \small{i) $\Pr[L=4]=0.81$} \\
    & \rotatebox{90}{$\quad$ \small{Arbitrary}}  & 
    \includegraphics[width=0.26\columnwidth]{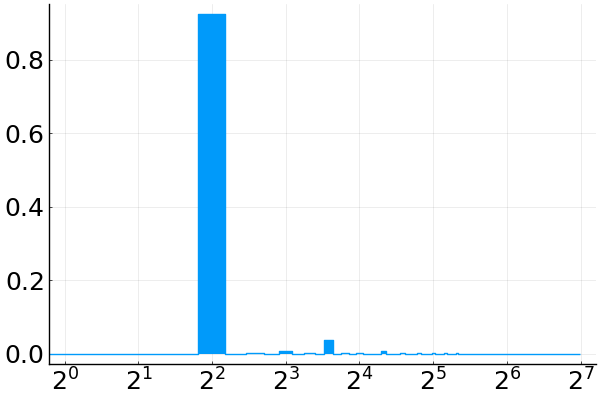} &
    \includegraphics[width=0.26\columnwidth]{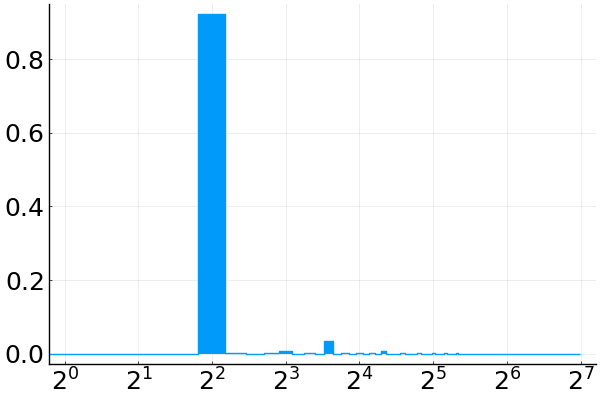} &
    \includegraphics[width=0.26\columnwidth]{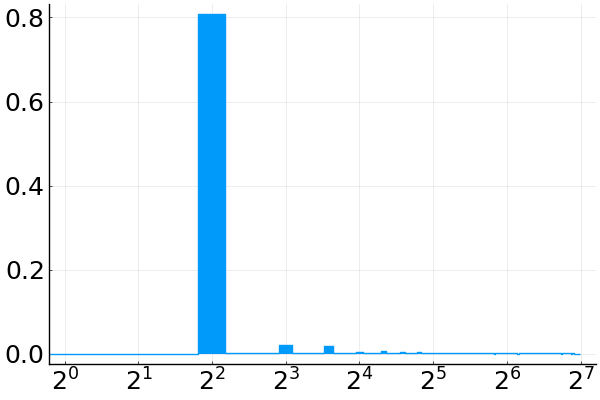}
    \end{tabular}
    \caption{\new{Probability of cycle length in synchronous structured CvHNNs when $A$ is symmetric and $B$ is an arbitrary matrix.}}
    \label{fig:ASymArb}
\end{figure}

Let us now examine the case where $A$ is an antisymmetric matrix, but there are no symmetry restrictions on the matrix $B$. Figure \ref{fig:ASymArb} illustrates the histogram of the cycle lengths for the resulting CvHNN models.

In contrast to the previous cases, we observe the prevalence of cycles of length $L=4$. Precisely, over $99\%$ of instances resulted in a cycle of length exactly $4$.
Despite the prevalence of cycles of length $L=4$, the CvHNN exhibited cycles of length as large as $L=16$ when the entries of $B$ are either positive or negative. Finally, we would like to point out that we observed large cycles when the entries of $B$ have arbitrary signs. 

\subsection{A is arbitrary and B is symmetric}

\begin{figure}[t]
    \begin{tabular}{ll|ccc}
    && \multicolumn{3}{c}{\underline{Arbitrary matrix $A$}} \\ 
    && \small{Positive} & \small{Negative} & \small{Arbitrary} \\ \hline
    && \small{a) $\Pr[L=8]=0.96$} 
    & \small{b) $\Pr[L=8]=0.96$} 
    & \small{c) $\Pr[L=4]=1.00$} \\
    \multirow{3}{*}{\hfill \rotatebox[origin=lB]{90}{\underline{Symmetric matrix $B$}}} & \rotatebox{90}{$\quad$ \small{Positive}} & 
    \includegraphics[width=0.26\columnwidth]{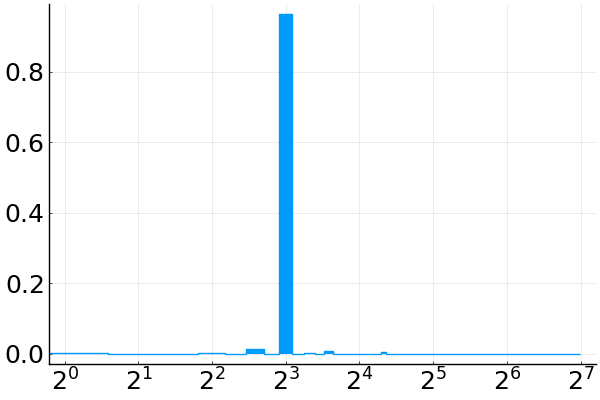} &
    \includegraphics[width=0.26\columnwidth]{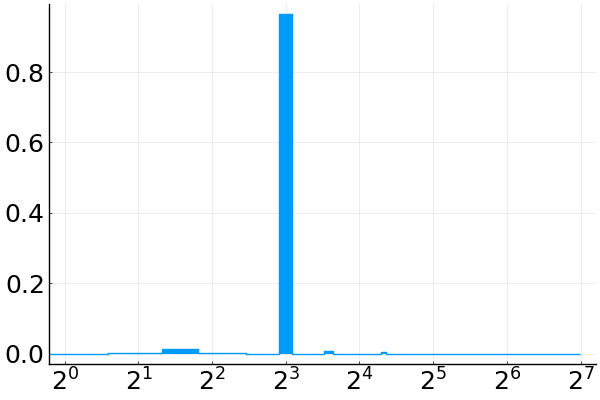} &
    \includegraphics[width=0.26\columnwidth]{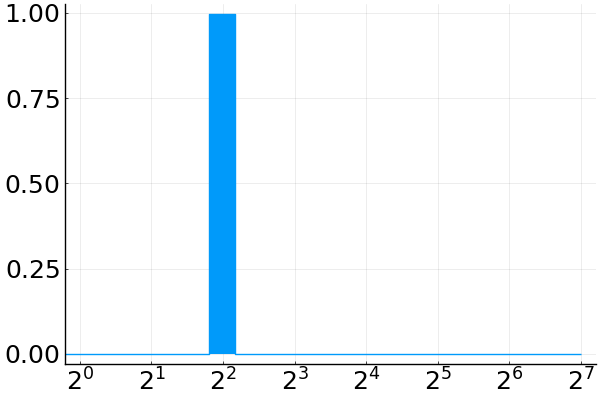} \\
    && \small{d) $\Pr[L=8]=0.96$} 
    & \small{e) $\Pr[L=8]=0.96$} 
    & \small{f) $\Pr[L=4]=1.00$} \\
    & \rotatebox{90}{$\quad$ \small{Negative}}  & 
    \includegraphics[width=0.26\columnwidth]{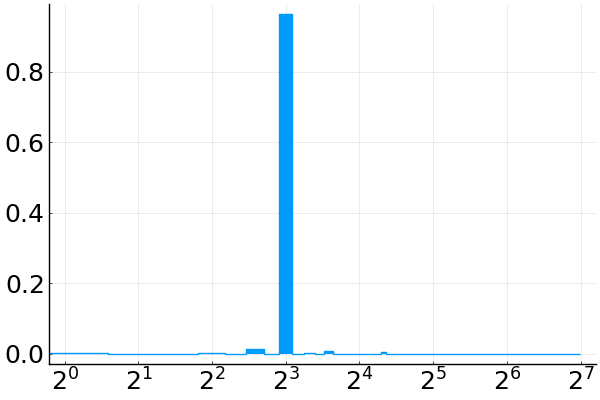} &
    \includegraphics[width=0.26\columnwidth]{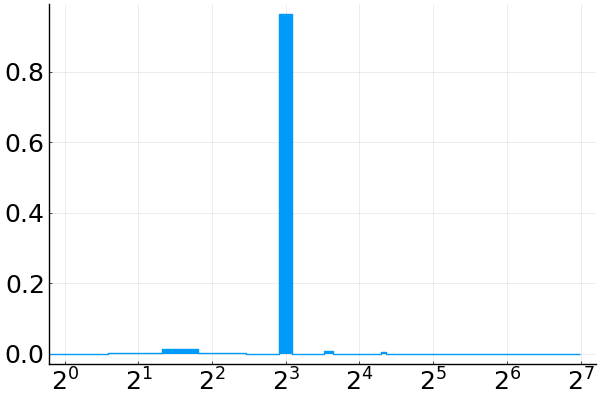} &
    \includegraphics[width=0.26\columnwidth]{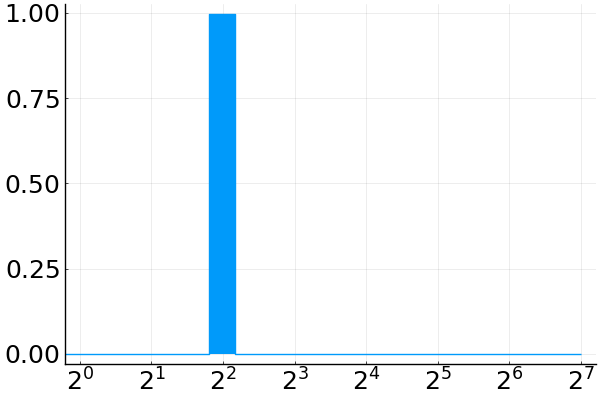} \\
    && \small{g) $\Pr[L=1]=0.98$} & 
    \small{h) $\Pr[L=2]=0.98$} & 
    \small{i) $\Pr[L=4]=0.73$} \\
    & \rotatebox{90}{$\quad$ \small{Arbitrary}}  & 
    \includegraphics[width=0.26\columnwidth]{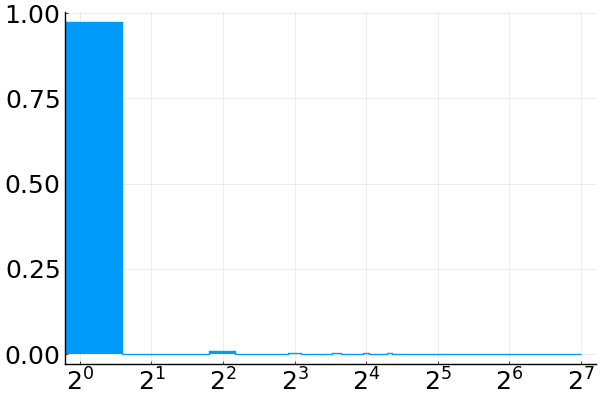} &
    \includegraphics[width=0.26\columnwidth]{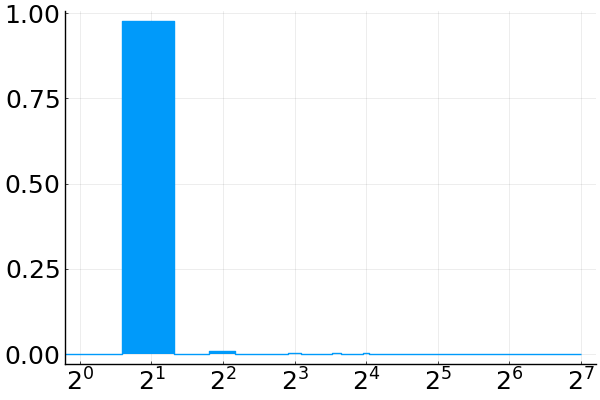} &
    \includegraphics[width=0.26\columnwidth]{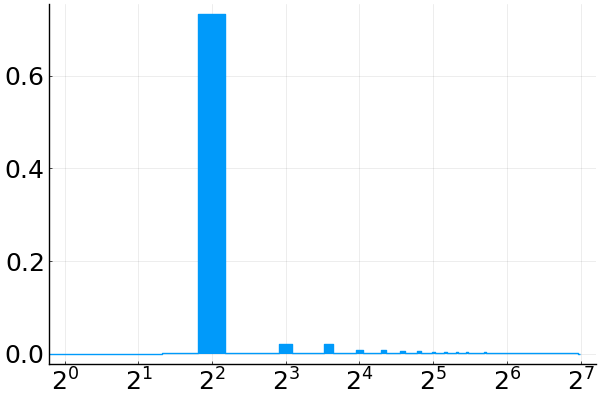}
    \end{tabular}
    \caption{\new{Probability of cycle length in synchronous structured CvHNNs when $A$ is arbitrary and $B$ is a symmetric matrix.}}
    \label{fig:ArbSym}
\end{figure}

Let us now consider the case where matrix \( A \) has no symmetry restrictions, but matrix \( B \) is symmetric. The outcomes of this scenario are illustrated in Figure \ref{fig:ArbSym}. Similar to the situation where \( A \) is symmetric and \( B \) is arbitrary, the complex-valued Hopfield Neural Network (HNN) has a high probability of stabilizing in a cycle of length \( L = 8 \) when constraints are placed on the signs of both \( A \) and \( B \). When the entries of \( B \) have no sign restrictions while \( A \) is either entirely positive or negative, the network often stabilizes in cycles of length \( L = 4 \). Furthermore, when the entries of \( A \) can take arbitrary signs but \( B \) is strictly positive or negative, the network typically settles at an equilibrium or stabilizes in a cycle of length \( L = 2 \). Finally, we have observed extremely large cycles occurring when both \( A \) and \( B \) have no sign constraints.

\subsection{A is arbitrary and B is antisymmetric}

\begin{figure}[t]
    \begin{tabular}{ll|ccc}
    && \multicolumn{3}{c}{\underline{Arbitrary matrix $A$}} \\ 
    && \small{Positive} & \small{Negative} & \small{Arbitrary} \\ \hline
    && \small{a) $\Pr[L=1]=0.99$} 
    & \small{b) $\Pr[L=2]=1.00$} 
    & \small{c) $\Pr[L=2]=0.76$} \\
    \multirow{3}{*}{\hfill \rotatebox[origin=lB]{90}{\underline{Antisymmetric matrix $B$}}} & \rotatebox{90}{$\quad$ \small{Positive}} & 
    \includegraphics[width=0.26\columnwidth]{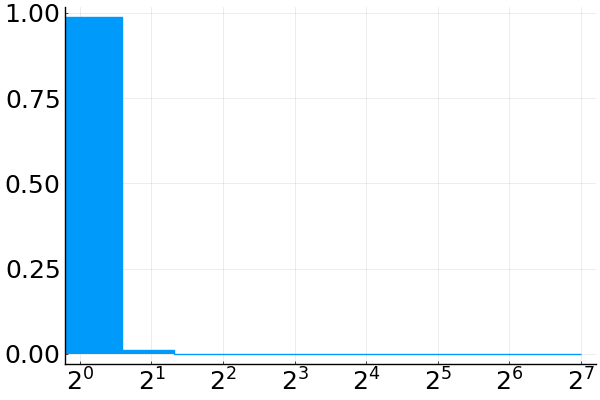} &
    \includegraphics[width=0.26\columnwidth]{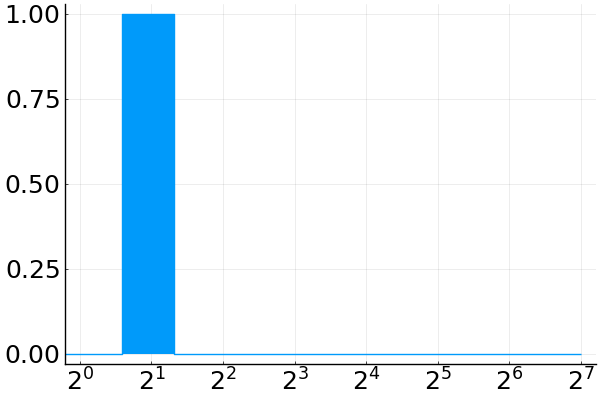} &
    \includegraphics[width=0.26\columnwidth]{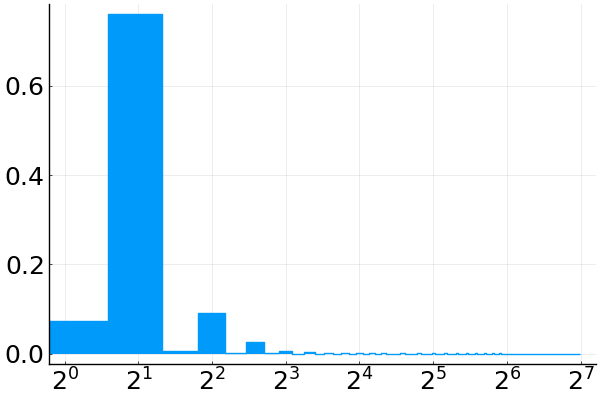} \\
    && \small{d) $\Pr[L=1]=0.99$} 
    & \small{e) $\Pr[L=2]=1.00$} 
    & \small{f) $\Pr[L=2]=0.76$} \\
    & \rotatebox{90}{$\quad$ \small{Negative}}  & 
    \includegraphics[width=0.26\columnwidth]{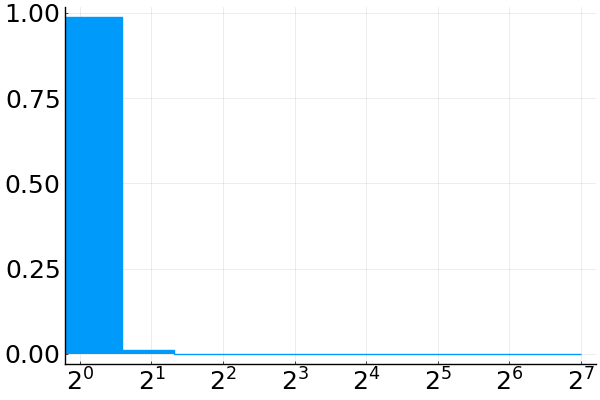} &
    \includegraphics[width=0.26\columnwidth]{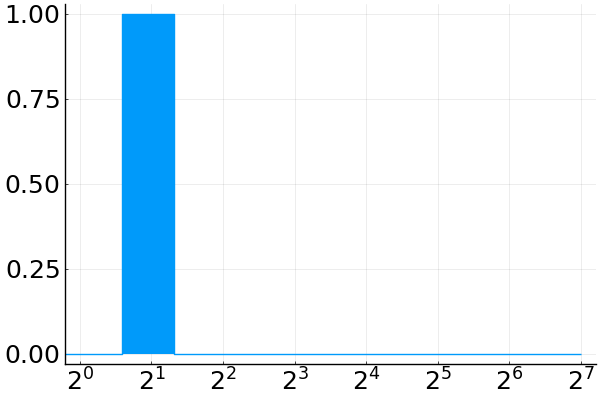} &
    \includegraphics[width=0.26\columnwidth]{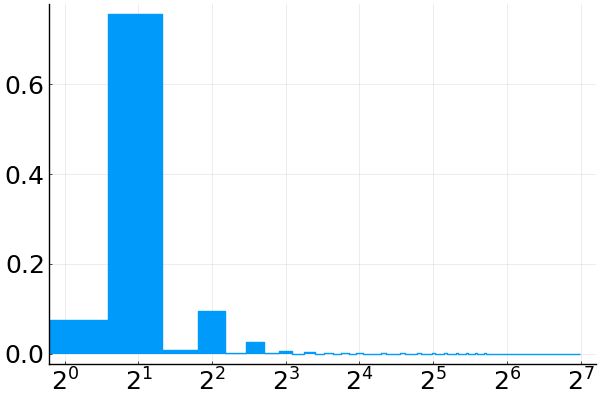} \\
    && \small{g) $\Pr[L=1]=0.99$} & 
    \small{h) $\Pr[L=2]=1.00$} & 
    \small{i) $\Pr[L=2]=0.86$} \\
    & \rotatebox{90}{$\quad$ \small{Arbitrary}}  & 
    \includegraphics[width=0.26\columnwidth]{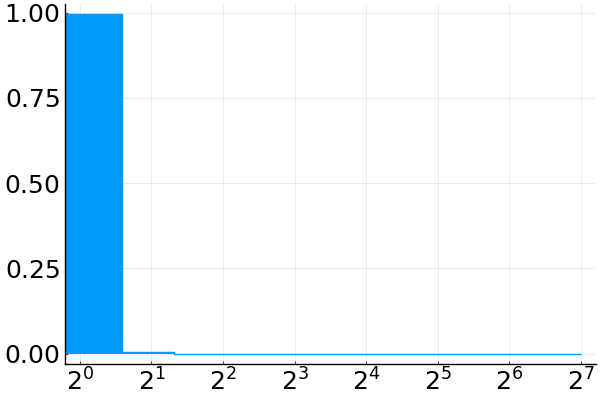} &
    \includegraphics[width=0.26\columnwidth]{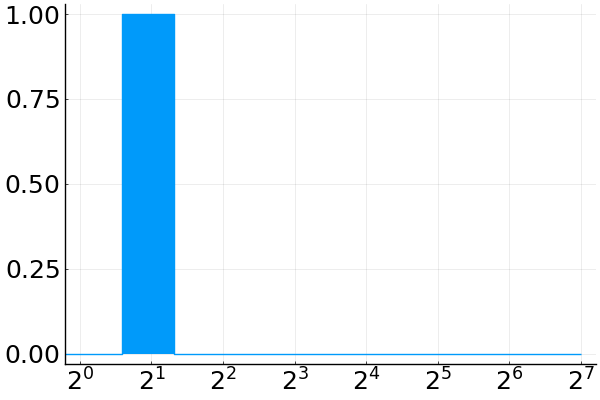} &
    \includegraphics[width=0.26\columnwidth]{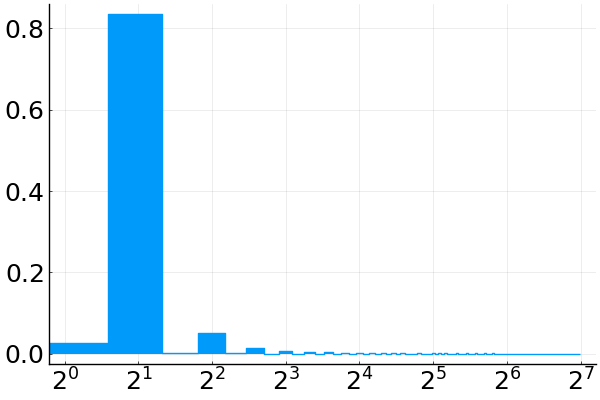}
    \end{tabular}
    \caption{\new{Probability of cycle length in synchronous structured CvHNNs when $A$ is arbitrary and $B$ is a antisymmetric matrix.}}
    \label{fig:ArbASym}
\end{figure}

The dynamics of the CvHNN is dominated by the sign of $A$ when $B$ is an antisymmetric matrix. Precisely, the network usually settles at an equilibrium when the entries of $A$ are all positive. If the entries of $A$ are all negative, the network often stabilizes in a cycle of length $L=2$. Finally, although the dynamics of the CvHNN is dominated by cycles of length $L=2$, we observed extremely large cycles when no sign restriction is imposed in $A$.

\subsection{A is arbitrary and B is arbitrary}

\begin{figure}[t]
    \begin{tabular}{ll|ccc}
    && \multicolumn{3}{c}{\underline{Arbitrary matrix $A$}} \\ 
    && \small{Positive} & \small{Negative} & \small{Arbitrary} \\ \hline
    && \small{a) $\Pr[L=8]=0.98$} 
    & \small{b) $\Pr[L=8]=0.98$} 
    & \small{c) $\Pr[L=4]=0.99$} \\
    \multirow{3}{*}{\hfill \rotatebox[origin=lB]{90}{\underline{Arbitrary matrix $B$}}} & \rotatebox{90}{$\quad$ \small{Positive}} & 
    \includegraphics[width=0.26\columnwidth]{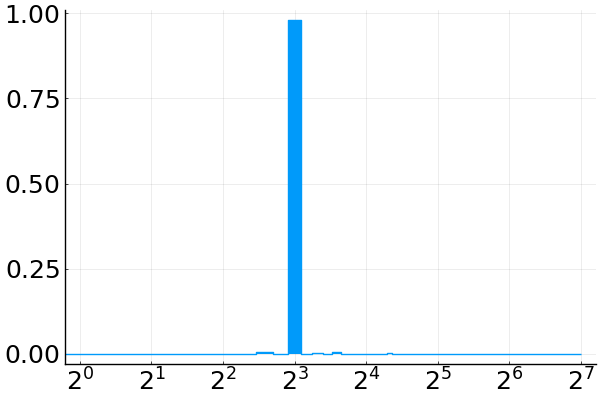} &
    \includegraphics[width=0.26\columnwidth]{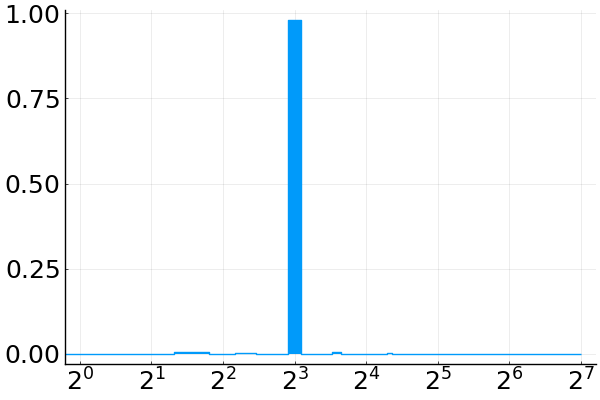} &
    \includegraphics[width=0.26\columnwidth]{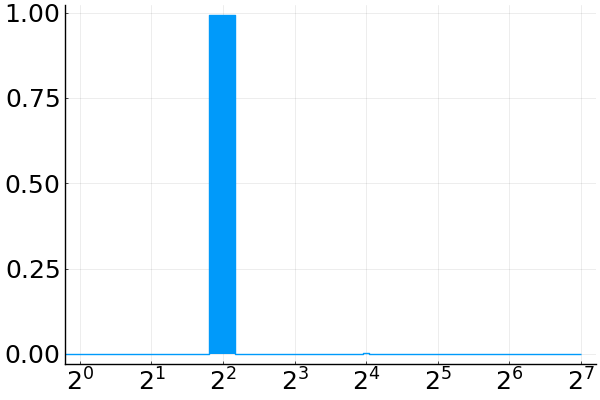} \\
    && \small{d) $\Pr[L=8]=0.98$} 
    & \small{e) $\Pr[L=8]=0.98$} 
    & \small{f) $\Pr[L=4]=0.99$} \\
    & \rotatebox{90}{$\quad$ \small{Negative}}  & 
    \includegraphics[width=0.26\columnwidth]{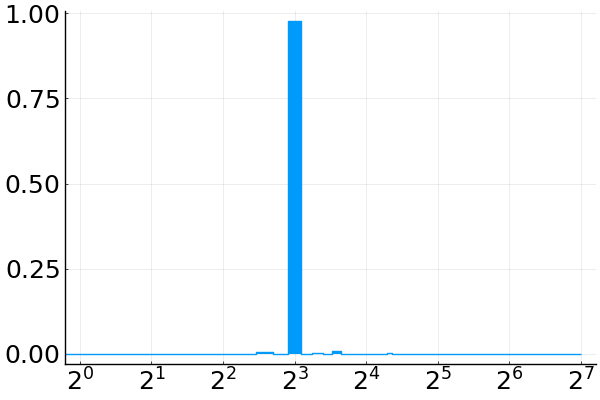} &
    \includegraphics[width=0.26\columnwidth]{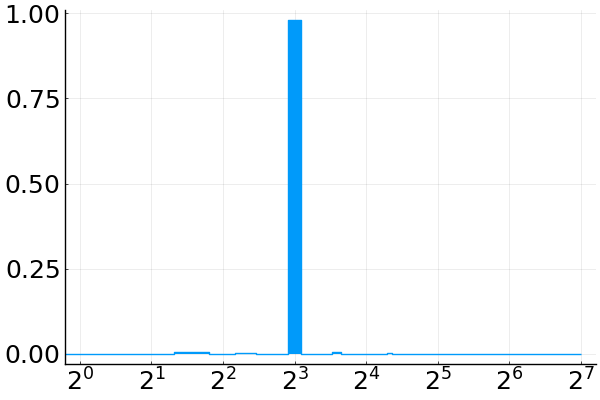} &
    \includegraphics[width=0.26\columnwidth]{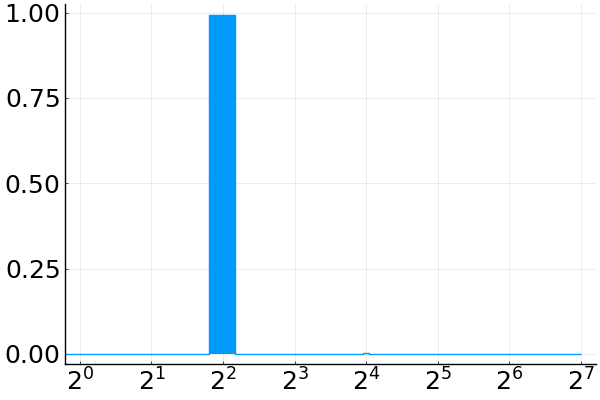} \\
    && \small{g) $\Pr[L=1]=0.99$} & 
    \small{h) $\Pr[L=2]=0.99$} & 
    \small{i) $\Pr[L=4]=0.02$} \\
    & \rotatebox{90}{$\quad$ \small{Arbitrary}}  & 
    \includegraphics[width=0.26\columnwidth]{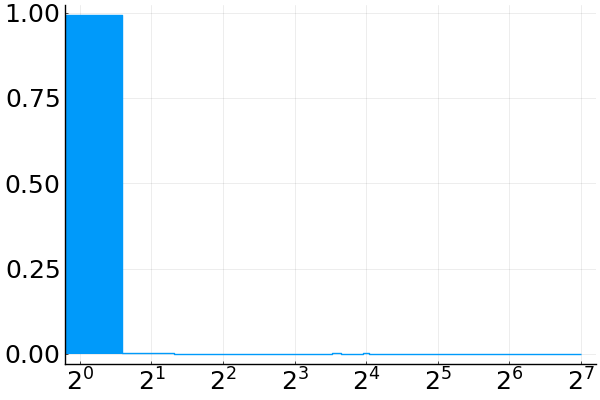} &
    \includegraphics[width=0.26\columnwidth]{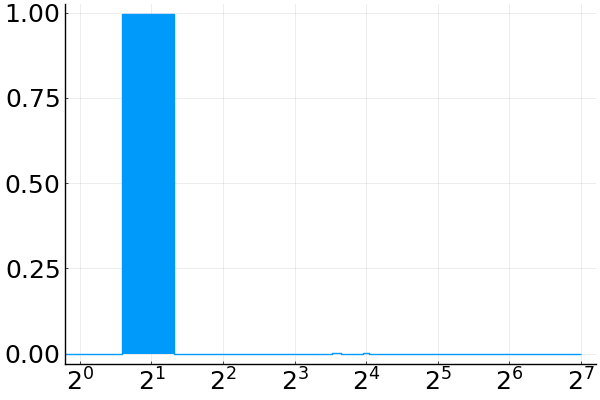} &
    \includegraphics[width=0.26\columnwidth]{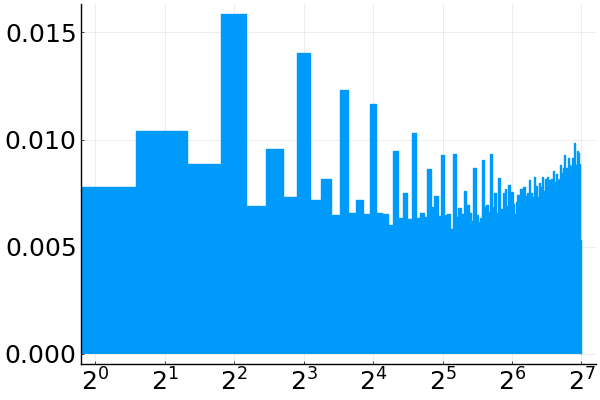}
    \end{tabular}
    \caption{Cycle lengths when $A$ is arbitrary and $B$ is a antisymmetric matrix.}
    \label{fig:ArbArb}
\end{figure}

Finally, let us consider the case in which both $A$ and $B$ are arbitrary, i.e., no symmetric restrictions are imposed in these matrices. The outcome of this case is depicted in Figure \ref{fig:ArbArb}.
Interestingly, the CvHNN usually stabilizes in a cycle of length $L=8$ if the matrices $A$ and $B$ are either positive or negative. On the one hand, if the entries of $A$ have arbitrary signs but $B$ is either positive or negative, the network often stabilizes in a cycle of length $L=4$. On the other hand, if the entries of $B$ have no sign constrain but $A$ is either positive or negative, the network frequently settles at an equilibrium or stabilizes in a cycle of length $L=2$. Finally, many cycle lengths, including extremely large cycles, are observed when $A$ and $B$ have no sign constraints.

\subsection{Conclusions from the computational experiments with synchronous CvHNNs} \label{ssec:summary}

\new{
The computational simulations presented in this subsection clearly demonstrate that there is no universal correlation between cycle length and the symmetry structure of the weight matrix \( M \). However, it is evident that cycles of lengths \( L = 4 \), \( L = 8 \), and \( L < 2 \) frequently arise when specific constraints are applied to the signs of the matrices \( A \) or \( B \).
}

\section{Complex-Valued Hopfield Networks with Synpatic Weight Matrix in Polar Form} \label{sec:polar}

In the previous sections, the synaptic weight matrix $M$ was represented as the sum of its real and imaginary parts. This section presents results by expressing the matrix $M$ in polar form.

Recall that the polar representation of a complex number includes both a \textit{magnitude} component and a \textit{phase} component. Thus, we can represent the synaptic weight matrix \( M \) using two matrices: \( G \) for the magnitude and \( P \) for the phase, where \( G_{ij} \) and \( P_{ij} \) represent the magnitude and phase of the entry \( M_{ij} \), respectively. Mathematically, the matrix $M$ can be expressed by
\begin{equation}
M_{ij} = G_{ij} e^{\ii P_{ij}}, \quad \forall i,j=1,\ldots,N.    
\end{equation}
Like Section \ref{sec:experiments}, we examine the dynamics of the synchronous CvHNNs based on the characteristics of the matrices \( G \) and \( P \). Figure \ref{fig:Polar} illustrates the dynamics of the CvHNN models, where the entries of \( G \) and \( P \) are drawn from uniform distributions over the intervals \([0,1]\) and \([- \pi, \pi]\), respectively. Additionally, Figure \ref{fig:Polar} displays the probabilities of the most frequent cycle lengths.

Note from Figure \ref{fig:Polar} that the phase plays a key role in the dynamics of the CvHNN models. The network stabilizes in cycles with many different lengths, including extremely large cycle lengths, when the phase matrix is symmetric or arbitrary. In contrast, the complex-valued network usually settles at an equilibrium or stabilizes in a cycle of lengths $L= 2$ or $L=4$ when the phase matrix $P$ is antisymmetric. In particular, note that the matrix $M$ is Hermitian when $G$ is symmetric and $P$ is antisymmetric. Thus, from Theorem \ref{thm:Hermitian}, the network stabilizes in a cycle of length $L \leq 2$. When both $G$ and $P$ are antisymmetric, we have 
$\bar{M}_{ji} = G_{ji} e^{-\ii P_{ji}} = -G_{ij} e^{\ii P_{ij}} = -M_{ij}$. Therefore, the matrix $M$ is skew-Hermitian and, from Theorem \ref{thm:Skew-Hermitian}, the CvHNN stabilizes in a cycle of length $L=4$. Finally, when $P$ is antisymmetric but no symmetry restriction is imposed on $G$, the weights satisfy $\text{phase}(M_{ji})=\text{phase}(M_{ij})$. In this case, most instances stabilized in a cycle of length $L<2$. However, we observed a few cycles of lengths $L>2$. Namely, we observed nine instances with cycles of length $L=4$ and one instance with cycle of length $L=28$, the largest obtained cycle.

\begin{figure}[t]
    \begin{tabular}{ll|ccc}
    && \multicolumn{3}{c}{\underline{Matrix $G$ (magnitude)}} \\ 
    && \small{Symmetric} & \small{Antisymmetric} & \small{Arbitrary} \\ \hline
    && \small{a) $\Pr[L=4]=0.02$} 
    & \small{b) $\Pr[L=4]=0.02$} 
    & \small{c) $\Pr[L=4]=0.02$} \\
    \multirow{3}{*}{\hfill \rotatebox[origin=lB]{90}{\underline{Matrix $P$ (phase)}}} & \rotatebox{90}{$\quad$ \small{Symmetric}} & 
    \includegraphics[width=0.26\columnwidth]{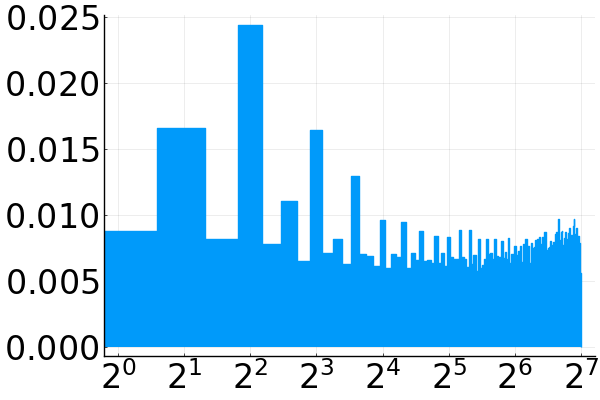} &
    \includegraphics[width=0.26\columnwidth]{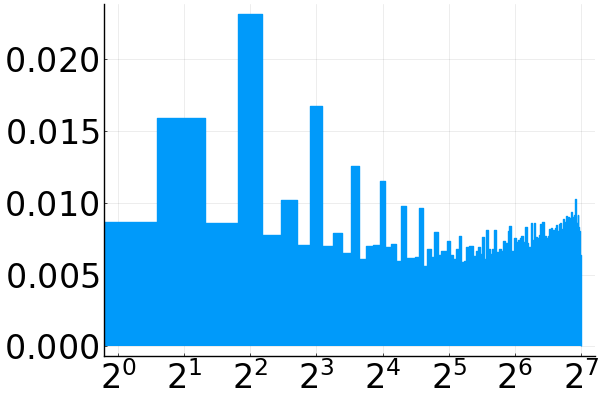} &
    \includegraphics[width=0.26\columnwidth]{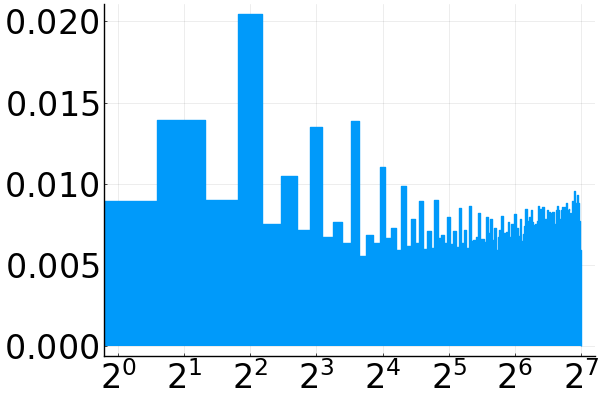} \\
    && \small{d) $\Pr[L=2]=0.77$} 
    & \small{e) $\Pr[L=4]=1.00$} 
    & \small{f) $\Pr[L=2]=0.51$} \\
    & \rotatebox{90}{$\quad$ \small{Antisymmetric}}  & 
    \includegraphics[width=0.26\columnwidth]{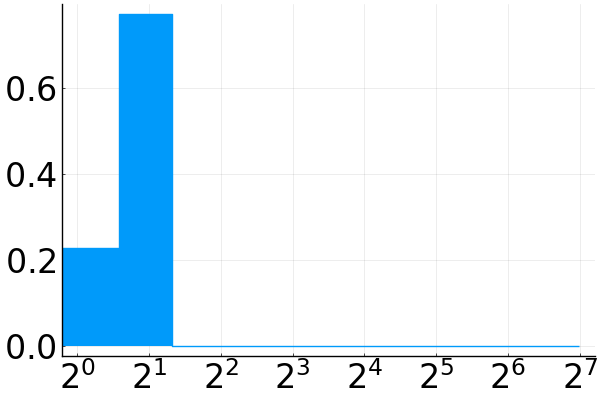} &
    \includegraphics[width=0.26\columnwidth]{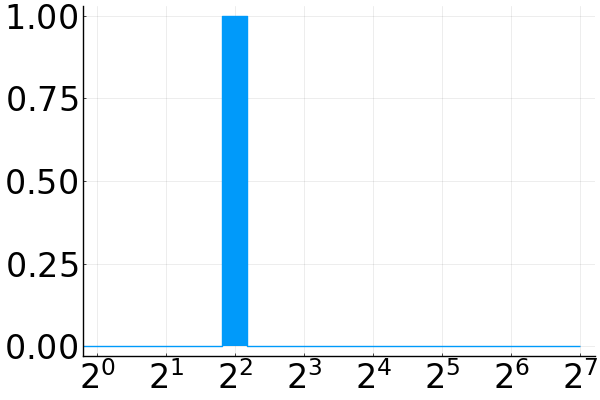} &
    \includegraphics[width=0.26\columnwidth]{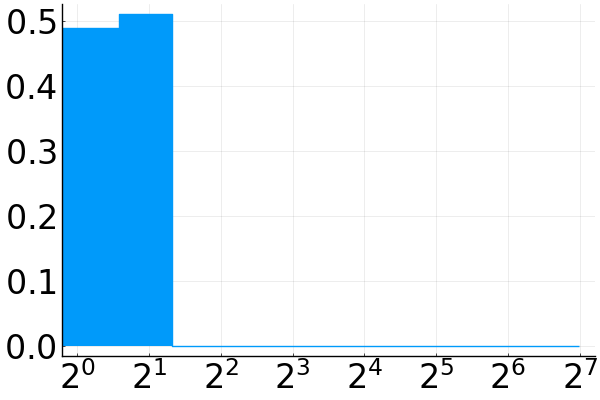} \\
    && \small{g) $\Pr[L=4]=0.02$} & 
    \small{h) $\Pr[L=8]=0.02$} & 
    \small{i) $\Pr[L=4]=0.02$} \\
    & \rotatebox{90}{$\quad$ \small{Arbitrary}}  & 
    \includegraphics[width=0.26\columnwidth]{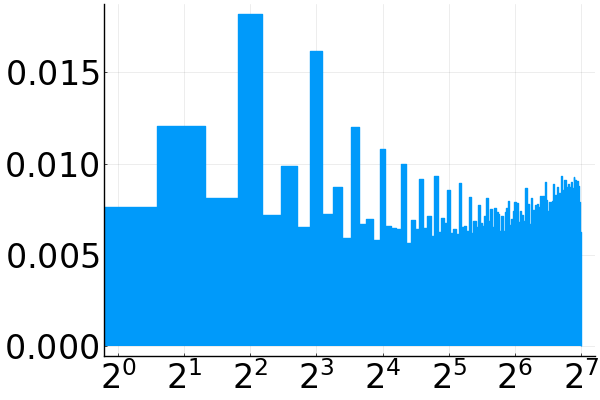} &
    \includegraphics[width=0.26\columnwidth]{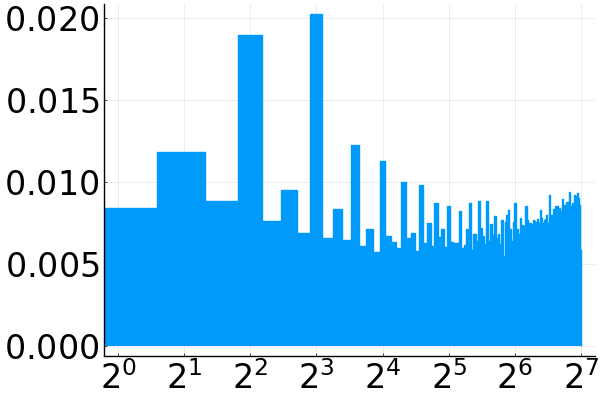} &
    \includegraphics[width=0.26\columnwidth]{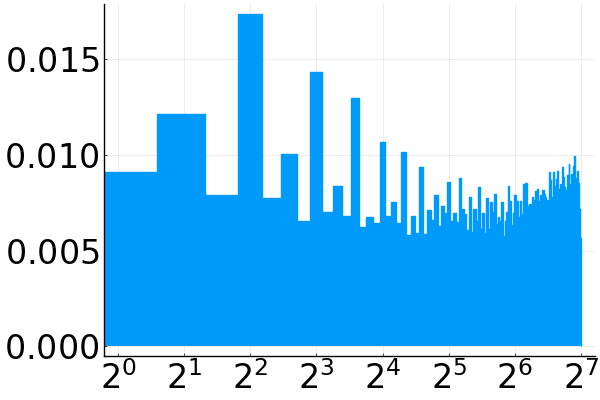}
    \end{tabular}
    \caption{\new{Probability of cycle length in synchronous structured CvHNNs obtained using the polar representation of $M$.}}
    \label{fig:Polar}
\end{figure}

\section{Conclusion} \label{sec:concluding}

\new{In this paper, we investigated the dynamics of CvHNNs with various structured weight matrices. Through theoretical analysis and computational experiments, we demonstrated that the behavior of these networks is significantly influenced by the properties of the synaptic weight matrix. Specifically, we introduced two new matrix forms—braided Hermitian and braided skew-Hermitian—and showed that networks employing these matrices exhibit periodic cycles of length eight when operated in parallel mode.}

\new{
Our computational experiments also revealed that the cycle lengths of the network states change according to the symmetry and sign constraints of the weight matrices. While certain configurations ensure stability or lead to short periodic cycles, others result in longer and more complex dynamics. The findings in this work can be used in conjunction with regular and multi-state associative memories, as detailed in \cite{DBLP:conf/ijcnn/GarimellaKG15}. Moreover, this paper contributes to a deeper understanding of structured CvHNNs and their potential applications in associative memory models, optimization problems, and neural computation. 
}

\new{Future research can build upon this work by exploring learning mechanisms for structured CvHNNs. Additionally, examining their role in modern machine learning architectures, such as attention-based networks and transformers, could yield new insights into efficient memory and recall mechanisms.}

\new{Moreover, in recent decades, numerous studies have developed hypercomplex-valued versions of neural network architectures. Evidence suggests that these models demonstrate superior performance in specific tasks compared to their real-valued counterparts \cite{minemoto17,Vieira2020ExtremeAuto-Encoding,Valle2020Hypercomplex-ValuedNetworks,Vieira2022AMachines}. Some researchers have extended the associative memory dynamics to higher-dimensional algebras, such as quaternions and octonions \cite{castro20nn,Garimella_2017}, and have presented results on convergence. Investigating structured hypercomplex-valued neural networks is also a promising area for future research.}


\end{document}